%% file: CMIGAN.tex
\documentclass[letterpaper]{article}
\usepackage{uai2020}
\usepackage[margin=1in]{geometry}
\usepackage{amsmath}
\usepackage{url}  %Required
\usepackage{graphicx}  %Required
 \usepackage{float}
\usepackage{subcaption}
\usepackage{hyperref}
\usepackage{amsfonts}
\usepackage{makecell}
\usepackage{multirow}
\usepackage{bbm}
\usepackage{color}
\usepackage{xr}
\usepackage{amsthm}

\newtheorem{theorem}{Theorem}

\usepackage{xcolor}
\usepackage[round]{natbib} % omit 'round' option if you prefer square brackets
\setcitestyle{authoryear, open={(},close={)},citesep={;},aysep=}
\usepackage{times}

\usepackage{algpseudocode,algorithm,algorithmicx}
\usepackage{natbib}

\newcommand*\Let[2]{\State #1 $\gets$ #2}
\newcommand{\printfnsymbol}[1]{%
  \textsuperscript{\@fnsymbol{#1}}%
}
\algrenewcommand\algorithmicrequire{\textbf{Precondition:}}
\algrenewcommand\algorithmicensure{\textbf{Postcondition:}}

\let\oldFootnote\footnote
\newcommand\nextToken\relax

\renewcommand\footnote[1]{%
    \oldFootnote{#1}\futurelet\nextToken\isFootnote}

\newcommand\isFootnote{%
    \ifx\footnote\nextToken\textsuperscript{,}\fi}
\title{C-MI-GAN : Estimation of Conditional Mutual Information \\ Using MinMax Formulation}

\par{\author{Arnab Kumar Mondal\thanks{Affiliated with IIT Delhi, India.} \\
{\small anz188380@cse.iitd.ac.in}
\And
Arnab Bhattacharjee \footnotemark[1]\\
{\small arnab.bhattacharjee@uqidar.iitd.ac.in}
\And
Sudipto Mukherjee\thanks{Affiliated with University of Washington, USA.}\\
{\small sudipm@uw.edu}
\AND
Prathosh AP \footnotemark[1]\\
\small  prathoshap@ee.iitd.ac.in
\And
Sreeram Kannan \footnotemark[2]\\
{\small ksreeram@uw.edu}
\And
Himanshu Asnani \thanks{Affiliated with TIFR Mumbai, India.}\\
{\small himanshu.asnani@tifr.res.in}
}}

\begin{document}

\maketitle

\input{content}

{\small\input{ack}}

{\small
\bibliography{Ref}
}

\input{supp.tex}

\end{document}

%% file: content.tex
\begin{abstract}
Estimation of information theoretic quantities such as mutual information and its conditional variant has drawn interest in recent times owing to their multifaceted applications. Newly proposed neural estimators for these quantities have overcome severe drawbacks of classical $k$NN-based estimators in high dimensions. In this work, we focus on conditional mutual information (CMI) estimation by utilizing its formulation as a \textit{minmax} optimization problem. Such a formulation leads to a joint training procedure similar to that of generative adversarial networks. We find that our proposed estimator provides better estimates than the existing approaches on a variety of simulated datasets comprising linear and non-linear relations between variables. As an application of CMI estimation, we deploy our estimator for conditional independence (CI) testing on real data and obtain better results than state-of-the-art CI testers.
\end{abstract}

\section{INTRODUCTION}

%\subsection{Motivation}
Quantifying the dependence between random variables is a quintessential problem in data science \citep{renyi1959measures, joe1989relative, fukumizu2008kernel}. A widely used measure across statistics is the Pearson correlation and partial correlation. Unfortunately, these measures can capture and quantify only linear relation between variables and do not extend to non-linear cases. The field of information theory \citep{cover2012elements} gave rise to multiple functionals of data density to capture the dependence between variables even in non-linear cases. Two noteworthy quantities of widespread interest are the mutual information (MI) and conditional mutual information (CMI).  %Assuming that all the distributions admit densities, the CMI can be defined as :
%\[
%I(X;Y|Z)=\iiint p(x,y,z)\log\frac{p(x, y, z)}{p(x, z) p(y|z)} %dxdydz
%\]

In this work, we focus on estimating CMI, a quantity which provides the degree of dependence between two random variables $X$ and $Y$ given a third variable $Z$. CMI provides a strong theoretical guarantee that $I(X;Y|Z) = 0 \iff X \perp Y |Z$. So, one motivation for estimating CMI is its use in conditional independence (CI) testing and detecting causal associations. CI tester built using $k$NN based CMI estimator coupled with a local permutation scheme (\cite{rungeCMIkNN}) was found to be better calibrated than the kernel tests. CMI was used for detecting and quantifying causal associations in spike trains data from neuron pairs \citep{li2011spike}. \cite{rungeAAAS} demonstrate how CMI estimator can be combined with a causal recovery algorithm to identify causal links in a network of variables. 

Apart from its use in CI testers, CMI has found diverse applications in feature selection, communication, network inference and image processing. Selecting features iteratively so that the information is maximized given already selected features was the basis for Conditional Mutual Information Maximization (CMIM) criterion in \cite{fleuret2004fast}. This principle was applied by \cite{wang2004feature} for text categorization, where the number of features are quite large. Efficient methods for CMI based feature selection involving more than one conditioning variable was developed by \cite{shishkin2016nips}. In the field of communications, \cite{yang2007mimo} maximized CMI between target and reflected waveforms for optimal radar waveform design. For learning gene regulatory network, in \cite{zhang2012grn} CMI was used as a measure of dependence between genes. A similar approach was adapted for protein modulation in \cite{giorgi2014protein}. Finally, \cite{loeckx2009nonrigid} used CMI as a similarity metric for non-rigid image registration. Given the widespread use of CMI as a measure of conditional dependence, there is a pressing need to accurately estimate this quantity, which we seek to achieve in this paper.

\section{RELATED WORK}

%\subsection{Prior Art}
 One of the simplest methods for estimating MI (or CMI) could be based on the binning of the continuous random variables, estimating probability densities from the bin frequencies and plugging it in the expression for MI (or CMI). Kernel methods, on the other hand, estimate the densities using suitable kernels. The most widely used estimator of MI, the KSG estimator, is based on $k$ nearest neighbor statistics \citep{kraskov2004} and has been shown to outperform binning or kernel-based methods. KSG is based on expressing MI in terms of entropy
\begin{equation}
I(X;Y) = H(X) + H(Y) - H(X;Y) 
\end{equation}
The entropy estimation follows from $H(X) = - N^{-1} \sum_i \log \widehat{\mu(x_i)}$ \citep{kozachenko1987}. Distance of the $k$ nearest neighbors of point $x_i$ is used to approximate the density $\mu(x_i)$. The KSG estimator does not estimate each entropy term independently, but accounts for the appropriate bias correction terms in the overall estimation. It ensures that an adaptive $k$ is used for distances in marginal spaces $X$, $Y$ and for the joint space $(X, Y)$. Several later works studied the theoretic properties of the KSG estimator and sought to improve its accuracy \citep{gao2015efficient, gao2016breaking, gao2018demystifying, poczos2012nonparametric}. Since CMI can be expressed as a difference of two MI estimates, $I(X;Y|Z) = I(X;YZ) - I(X;Z)$, KSG estimator could be used for CMI as well. Even though KSG estimator enjoys the favorable property of consistency, its performance in finite sample regimes suffers from the curse of dimensionality. In fact, KSG estimator requires exponentially many samples for accurate estimation of MI \citep{gao2015efficient}. This limits its applicability in high dimensions with few samples. 

Deviating from the $k$NN-based estimation paradigm, \cite{belghazi} proposed a neural estimation of MI (referred to as MINE). This estimator is built on optimizing dual representations of the KL divergence, namely the Donsker-Varadhan \citep{donsker1975asymptotic} and the f-divergence representation \citep{nguyen2010estimating}. MINE is strongly consistent and scales well with dimensionality and sample size. However, recent works found the estimates from MINE to have high variance \citep{poole, oord2018representation} and the optimization to be unstable in high dimensions \citep{mukherjee2019ccmi}. To counter these issues, variance reduction techniques were explored in \cite{song2019iclr}.

While for the estimation of MI we need to perform the trivial task of drawing samples from the marginal distribution, CMI estimation adds another layer of intricacy to the problem. For the above approaches to work for CMI, one needs to obtain samples from the conditional distribution. In \cite{mukherjee2019ccmi}, the authors separate the problem of estimating CMI into two stages by first estimating the conditional distribution and then using a divergence estimator. However, being coupled with an initial conditional distribution sampler, this technique is limited by the goodness of the conditional samplers and thus may be sub-optimal. Even when CMI is obtained as a difference of two separate MI estimates (CCMI estimator in \cite{mukherjee2019ccmi}), there is no guarantee that the bias values would be same from both MI terms, thereby leading to incorrect estimates. Based on these observations, in this paper, we attempt to estimate CMI using a joint training procedure involving a min-max formulation devoid of explicit conditional sampling. 

%\subsection{Our Contributions}
The main contributions of our paper are as follows:
\begin{itemize}
\item We formulate CMI as a \textit{minimax} optimization problem and illustrate how it can be estimated from joint training. The estimation process has similar flavor to adversarial training \citep{goodfellow2014GAN} and so the term C-MI-GAN (read ``See-Me-GAN'') is coined for the estimator.
\item We empirically show that estimates from C-MI-GAN are closer to the ground truth on an average as compared to the estimates of other CMI estimators.
\item We apply our estimator for conditional independence testing on a real flow-cytometry dataset and obtain better results than state-of-the-art CI Testers.  
\end{itemize} 

\section{PROPOSED METHODOLOGY}

\textbf{Information Theoretic Quantities.}  Let $X$, $Y$ and $Z$ be three continuous random variables that admit densities. The mutual information between two random variables $X$ and $Y$ measures the amount of dependence between them and is defined as 
\begin{equation}
I(X;Y)=\iint P_{XY}(x,y)\log\frac{P_{XY}(x, y)}{P_X(x) P_Y(y)} dxdy
\end{equation}
It can also be expressed in terms of the entropies as follows:
\begin{equation}
     I(X;Y) = H(X) - H(X|Y) = H(Y) - H(Y|X)
\end{equation}
Here $H(X)$ is the entropy\footnote{More precisely, differential entropy in case of continuous random variables.} and is given by $H(X) = - \int p(x) \, log \, p(x) \, dx$.
The above expression provides the intuitive explanation of how the information content changes when the random variable is alone versus when another random variable is given. 

The conditional mutual information extends this to the setting where a conditioning variables is present. The analogous expression for CMI is:
\begin{equation}
    I(X;Y|Z)=\iiint P_{XYZ}\log\frac{P_{XYZ}}{P_{XZ} P_{Y|Z}} dxdydz
    \label{def:cmi}
\end{equation}{}

In terms of the entropies, it can be expressed as follows:
\begin{align}
    I(X;Y|Z) &= H(X|Z) - H(X|Y,Z) \\
    &= H(Y|Z) - H(Y|X, Z)
\end{align}

Both MI and CMI are special cases of a statistical quantity called KL-divergence, which measures how different one distribution is from another. The KL-divergence between two distributions $P_X$ and $Q_X$ is as follows:
\begin{equation}
D_{KL}(P_{X} || Q_{X}) = \int P_X(x) \log \frac{P_X(x)}{Q_X(x)} dx
\end{equation}

% In terms of the KL-divergence, we can express MI and CMI as follows:
MI and CMI can be defined using KL-divergence as:
\begin{equation}
    I(X;Y) = D_{KL}(P_{XY} || P_X P _Y)
    \label{eqn:mi_dkl}
\end{equation}
\begin{equation}
    I(X;Y|Z) = D_{KL}(P_{XYZ}||P_{XZ}P_{Y|Z})
\end{equation}
This definition of MI (Equation \ref{eqn:mi_dkl}) tries to capture how much the given joint distribution is different from $X$ and $Y$ being independent (conditionally independent in case of CMI). Various estimators in the literature aim to utilize a particular expression of MI (or CMI), while avoiding computation of density functions explicitly. While KSG (\citep{kraskov2004}) is based on the summation of entropy terms, MINE (\citep{belghazi}) derives its estimates based on lower bounds of KL-divergence.

\textbf{Lower bounds of Mutual Information.} The following lower bounds of KL-divergence (hence also mutual information) were used in \cite{belghazi} for the MINE estimator.

Donsker-Varadhan bound : This bound is tighter and is given by : 
\begin{equation}
  D_{KL}(P||Q) = \sup_{R \in \mathcal{R}} (\mathbb{E}_P[R] - \log(\mathbb{E}_Q[e^R]))  \label{eq:dv_identity} 
\end{equation}

f-divergence bound : A slightly loose bound is given by the following relation :
\begin{equation}
D_{KL}(P||Q) = \sup_{R \in \mathcal{R}} (\mathbb{E}_P[R] - \mathbb{E}_Q[e^{R-1}])
\label{eq:f_div}
\end{equation}
The supremum in both the bounds (equation \ref{eq:dv_identity} and \ref{eq:f_div}) is over all functions $R \in \mathcal{R}$ such that the expectations are finite. MINE uses a neural network as a parameterized function $R_\beta$, which is optimized with these bounds. 

\subsection{MIN-MAX FORMULATION FOR CMI}\label{sec:cmi_min_max}

Building on top of these lower bounds, we further take resort to a variational form of conditional mutual information. Observing closely the expression for $I(X;Y|Z)$ in equation \ref{def:cmi}, we find that samples need to be drawn from $p(y|z)$ which is not available from given data $(X, Y, Z)$ directly. One approach used in \cite{mukherjee2019ccmi} is to learn $p(y|z)$ using a conditional GAN, $k$NN or conditional VAE. \textit{Can we combine this step directly with the lower bound maximization ?} 

We first note that the CMI expression can be upper bounded as follows :
\begin{equation}
    \begin{split}
        I(X;Y|Z) &= D_{KL}(P_{XYZ}||P_{XZ}P_{Y|Z}) \\
        &= D_{KL}(P_{XYZ}||P_{XZ}Q_{Y|Z}) \\
        &\qquad - D_{KL}(P_{Y|Z} || Q_{Y|Z}) \\
        &\le D_{KL}(P_{XYZ}||P_{XZ}Q_{Y|Z})
    \end{split}
    \label{cmi_upper_bound}
\end{equation}
since $D_{KL}(P||Q) \geq 0$.
In equation \ref{cmi_upper_bound} the equality is achieved when $Q_{Y|Z} = P_{Y|Z}$ and we can express $I(X;Y|Z)$ as
\begin{equation}
    \begin{split}
        I(X;Y|Z) = \mathop{\text{inf}}_{Q_{Y|Z}} D_{KL}(P_{XYZ}||P_{XZ}Q_{Y|Z})
    \end{split}
    \label{mi_minimization_formulation}
\end{equation}
Equation \ref{mi_minimization_formulation} coupled with the Donsker-Varadhan bound (equation \ref{eq:dv_identity}) leads to a min-max optimization for MI estimation as follows:
\begin{equation}
    \begin{split}
        I(X;Y|Z) &= \mathop{\text{inf}}_{Q_{Y|Z}}~\mathop{\text{sup}}_{R \in \mathcal{R}}\bigg(\mathop{\mathbb{E}}_{s\sim P_{XYZ}}[R(s)] \\
        &\qquad - \log\big(\mathop{\mathbb{E}}_{s\sim P_{XZ}Q_{Y|Z}}[e^{R(s)}]\big)\bigg)
    \end{split}
    \label{cmi_min_max_formulation}
\end{equation}
Equation \ref{cmi_min_max_formulation} offers a pragmatic approach for estimating CMI. Since neural nets are universal function approximators, it is a possibility to deploy one such network to approximate the variational distribution $Q_{Y|Z}$ and another for learning the regression function given by $R$. The following section provides a detailed narration of how to achieve this objective.

\subsection{C-MI-GAN}
\begin{figure}[!ht]
    \centering
    \includegraphics[keepaspectratio, width=\linewidth]{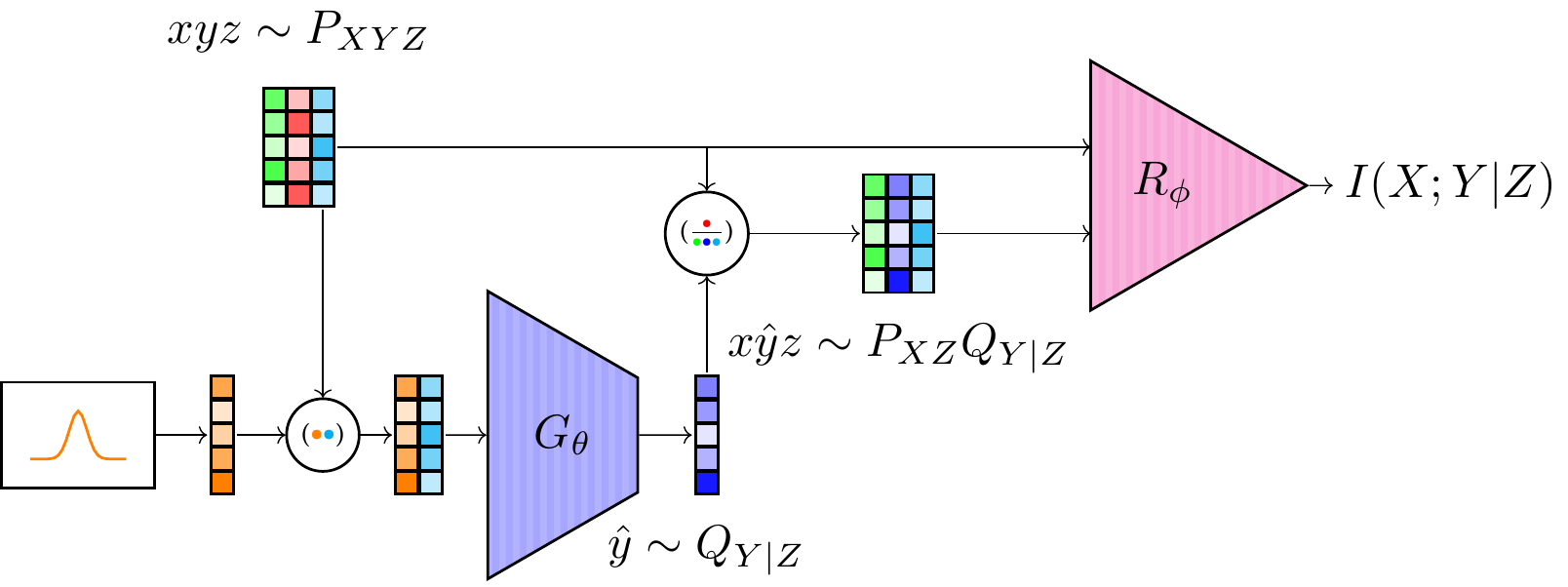}
    \caption{Block Diagram for C-MI-GAN (Best viewed in colour). Samples drawn from any simplistic noise distribution are concatenated with the samples from the marginal $P_{Z}$ and fed to the generator as input. The generated samples from the variational distribution $Q_{Y|Z}$ are then concatenated with samples from $P_{XZ}$ and given as input to the regression network along with samples from $P_{XYZ}$. $I(X;Y|Z)$ is obtained by negating the loss of the trained regression network.}
    \label{fig:cmigan_bd}
    \vspace{-5mm}
\end{figure}{}
To begin with, we elaborate different components of the proposed estimator - C-MI-GAN. As depicted in Figure \ref{fig:cmigan_bd}, the variational distribution, $Q_{Y|Z}$ is parameterized using a network denoted as $G_\theta$. In other words, $G_\theta$ is capable of sampling from the distribution $Q_{Y|Z}$, hence it is called the generator network. The regression network, $R_\phi$ parameterizes the function class on the R.H.S. of the Donsker-Varadhan identity (refer to equation \ref{eq:dv_identity}). Gaussian noise concatenated with conditioning variable $Z$ is fed as input to $G_\theta$. $R_\phi$ is trained with samples from $P_{XYZ}$ and $P_{XZ}Q_{Y|Z}$. During training, we optimize the regression network and the generator jointly using the objective function $h(Q_{Y|Z}, R)$ as defined below.
\begin{equation}
    \begin{split}
        h(Q_{Y|Z}, R) &= \mathop{\text{inf}}_{Q_{Y|Z}}~\mathop{\text{sup}}_{R \in \mathcal{R}}\Bigg(\int\limits_{s\sim P_{XYZ}}P_{XYZ}R(s)ds - \\
        &\qquad\log\bigg(\int\limits_{s\sim P_{XZ}Q_{Y|Z}}P_{XZ}Q_{Y|Z}e^{R(s)}ds\bigg)\Bigg)
    \end{split}{}
    \label{cmigan_loss_function}
\end{equation}{}
In each training loop, we optimize the parameters of $R_\phi$ and $G_\theta$, using a learning schedule and RMSProp optimizer. The detailed procedure is described in Algorithm \ref{alg:cmigan-train-loop}.
Upon successful completion of training of the joint network, $G_\theta$ starts generating samples from the distribution $P_{Y|Z}$ and the output of the regression network converges to $I(X;Y|Z)$.\par
Next, we formally show that the alternate optimization of $R_\phi$ and $G_\theta$ optimizes the objective function defined in equation \ref{cmigan_loss_function} and when the global optima is reached, the optimal value of the objective function coincides with CMI.
To start with, we derive the expression for the optimal regression network and subsequently show that the optimal value of the objective function coincides with CMI under the optimal regression network.

\begin{theorem}
    For a given generator, $G$ , the optimal regression network, $R^{*}$, is 
    \begin{equation}
        \begin{split}
            R^{*} = \log{\frac{P_{XYZ}}{P_{XZ}Q_{Y|Z}}} + c
        \end{split}{}
    \end{equation}{}
    Where $c$ is any constant.
    $P_{XYZ}$, $P_{XZ}$, and $Q_{Y|Z}$ denote the data distribution, marginal distribution and generator distribution respectively.
\end{theorem}
\begin{proof}\footnote{This proof assumes $P_{XYZ}(s), P_{XZ}Q_{Y|Z}(s) > 0 ~\forall s$.}
For a given generator, $G$, the regression network's objective is to maximize the quantity $h(Q_{Y|Z}, R)$.\\
% Let, $P_{XY} = \mathbb{P}, P_{X}Q_{y} = \mathbb{Q}$ and $r_i = R(s_i)$, where, $s_i \sim \mathbb{P}$ with probability $\mathbb{P}_i$ or $s_i \sim \mathbb{Q}$ with probability $\mathbb{Q}_i$
\begin{equation}
    \begin{split}
        % h(Q_Y, R) &= \mathop{\mathbb{E}}_{s\sim P_{XY}}[R(s)] - \log\big(\mathop{\mathbb{E}}_{s\sim P_XQ_Y}[e^{R(s)}]\big) \\
        % &= \sum\limits_{i}\mathbb{P}_{i}r_i - \log\bigg(\sum\limits_{i}\mathbb{Q}_ie^{r_i}\bigg) \\ 
        h(Q_{Y|Z}, R) &= \int\limits_{s\sim P_{XYZ}}P_{XYZ}R(s)ds - \\
        &\hspace{-2mm}\log\bigg(\int\limits_{s\sim P_{XZ}Q_{Y|Z}}P_{XZ}Q_{Y|Z}e^{R(s)}ds\bigg)
    \end{split}{}
\end{equation}{}
For the optimum regression network, $R^*$,
\begin{equation}
    \begin{aligned}
        &\qquad\qquad\qquad\quad \frac{\partial h}{\partial R}\bigg\lvert_{R^*} = 0 \\
        &\implies P_{XYZ} - \frac{P_{XZ}Q_{Y|Z}e^{R^*}}{\int P_{XZ}Q_{Y|Z}e^{R^*}ds} = 0 \\
        &\implies \frac{P_{XZ}Q_{Y|Z}e^{R^*}}{P_{XYZ}} = \int P_{XZ}Q_{Y|Z}e^{R^*}ds = e^c \\
        &\implies R^{*} = \log{\frac{P_{XYZ}}{P_{XZ}Q_{Y|Z}}} + c
    \end{aligned}{}
\end{equation}{}
% Solving, the optimum regression network, $R^*$, for a given generator is, $$R^{*} = \log{\frac{P_{XY}}{P_XQ_Y}} + \log\mathop{\mathbb{E}}_{P_XQ_Y}e^{R}$$ 
\end{proof}{}
Now we show that with the optimal regression network $R^*$, CMI is obtained when the objective function achieves its minima.
\begin{theorem}
    $h(Q_{Y|Z}, R^*)$ achieves its minimum value $I(X;Y|Z)$, iff $Q_{Y|Z} = P_{Y|Z}$.
\end{theorem}
\begin{proof}
    \begin{equation}
    \begin{split}
        h(Q_{Y|Z}, R^*) &= \int\limits_{s}P_{XYZ}\log{\frac{P_{XYZ}}{P_{XZ}Q_{Y|Z}}}dxdydz - \\
        &\quad \log\bigg(\int\limits_{s}P_{XZ}Q_{Y|Z}e^{\log\frac{P_{XYZ}}{P_{XZ}Q_{Y|Z}}}dxdydz\bigg) \\
        &= \int\limits_{s}P_{XYZ}\log{\frac{P_{XYZ}}{P_{XZ}Q_{Y|Z}}}dxdy - \\
        &\qquad \int\limits_{s}P_{XYZ}dxdydz \\
        &= \int\limits_{s}P_{XYZ}\log{\frac{P_{XYZ}}{P_{XZ}P_{Y|Z}}}dxdydz + \\
        &\qquad \int\limits_{s}P_{Y|Z}\log{\frac{P_{Y|Z}}{Q_{Y|Z}}}dy \\
        &= I(X;Y|Z) + D_{KL}(P_{Y|Z}||Q_{Y|Z})
        % h(Q_Y, R^*) &= \sum\limits_{i}\mathbb{P}_{i}r_i^* - \log\bigg(\sum\limits_{i}\mathbb{Q}_ie^{r_i^*}\bigg) \\
        % &= \sum\limits_{i}\mathbb{P}_{i}\log\frac{\mathbb{P}_i}{\mathbb{Q}_i} + c - \log\bigg(\sum\limits_{i}\mathbb{Q}_ie^{(\log\frac{\mathbb{P}_i}{\mathbb{Q}_i} + c)}\bigg) \\
        % &= \sum\limits_{x}\sum\limits_{y}P_{XY}\log\frac{P_{XY}}{P_XQ_Y} + c - \log e^c \\
        % &= \sum\limits_{x}\sum\limits_{y}P_{XY}\log\frac{P_{XY}}{P_XP_Y} + \sum\limits_{y}P_{Y}\log\frac{P_Y}{Q_Y} \\
        % &= I(X;Y) + D_{KL}(P_Y||Q_Y)
    \end{split}{}
\end{equation}{}
Since $D_{KL}(P_{Y|Z}||Q_{Y|Z}) \ge 0$, when $P_{Y|Z}=Q_{Y|Z}$, $D_{KL}(P_{Y|Z}||Q_{Y|Z}) = 0$ and the minimum value is $h(P_{Y|Z}, R^*) = I(X;Y|Z)$.
\end{proof}

The alternate optimization of $R_\phi$ and $G_\theta$ is similar to the generative adversarial networks (\cite{goodfellow2014GAN}, \cite{fgan}), in that both are trained using similar adversarial training procedure. However, C-MI-GAN significantly differs from traditional GAN in the following sense:
\begin{itemize}
    \item The regression task is completely unsupervised as no target value is used to train the network.
    \item The proposed loss function to estimate CMI is foreign to traditional GAN literature.
    \item The binary discriminator in traditional GAN is replaced by a regression network, $R_\phi$ that estimates the CMI (refer to Figure \ref{fig:cmigan_bd}).
\end{itemize}{}

 \begin{algorithm}
    \caption{Pseudo code for C-MI-GAN \label{alg:cmigan-train-loop}}
    \hspace*{\algorithmicindent} \textbf{Inputs:} $\mathcal{D}=\{x^{(i)}, y^{(i)}, z^{(i)}\}_{i=1}^{m} \sim P_{XYZ}$ \\
    \hspace*{\algorithmicindent} \textbf{Outputs:} $I(X;Y|Z)$ \\
    \begin{algorithmic}[1]
        \Function{CMIGAN}{}
            \For{$r \gets 1 \textrm{ to } N$}
                \State Initialize $R_\phi$ and $G_\theta$
                \For{$i \gets 1 \textrm{ to } training\_steps$}
                    \For{$j \gets 1 \textrm{ to } reg\_training\_ratio$}
                        \State Shuffle $\mathcal{D} \sim P_{XYZ}$             
                        \Let{$\mathcal{D}_b$}{$\{x^{(k)}, y^{(k)}, z^{(k)}\}_{k=1}^{s} \sim P_{XYZ}$}
                        \Let{$noise$}{$\{n^{(k)}\}_{k=1}^{s} \sim \mathcal{N}(0, I_{d_n})$}
                        \Let{$\{y_\theta^{(k)}\}_{k=1}^{s}$}{$G_\theta(noise,Z_b)$}
                        \Let{$\hat{\mathcal{D}}_b$}{$\{x^{(k)}, y_\theta^{(k)}, z^{(k)}\}_{k=1}^{s}$}
                        \Let{$L_{reg}$}{$-\mathop{\mathbb{E}}_{\mathcal{D}_b}[R_\phi]+ \log(\mathop{\mathbb{E}}_{\hat{\mathcal{D}}_b}[e^{R_\phi}])$}
                        % \Let{$t$}{$ij+j$}
                        % \Let{$g_{\phi}(t)$}{$\nabla_\phi L_{reg}$}
                        % \Let{$\phi_{t}$}{$\phi_{t-1}-\frac{\eta_{reg} g_{\phi}(t)}{\sqrt{\rho g_{\phi}(t) + (1-\rho)g_{\phi}^2(t-1)+\epsilon}}$}
                        \State Minimize $L_{reg}$ and Update $\phi$
                    \EndFor
                    \Let{$L_{gen}$}{$-\log(\mathop{\mathbb{E}}_{\hat{\mathcal{D}}_b}[e^{R_\phi}])$}
                    % \Let{$g_{\theta}(i)$}{$\nabla_\theta L_{gen}$}
                    % \Let{$\theta_{i}$}{$\theta_{i-1}-\frac{\eta_{gen} g_{\theta}(i)}{\sqrt{\rho g_{\theta}(i) + (1-\rho)g_{\theta}^2(i-1)+\epsilon}}$}
                    \State Minimize $L_{gen}$ and Update $\theta$
                    % \If {$i \% lr\_schedule\_interval == 0$}
                    %     \Let{$\eta_{gen}$}{$\eta_{gen} \times decay\_rate$}
                    %     \Let{$\eta_{reg}$}{$\eta_{reg} \times decay\_rate$}
                    \State Update learning rate as per scheduler.
                    % \EndIf
                \EndFor
                \Let{$\hat{I}_n(X;Y|Z)$}{$-L_{reg}$}
            \EndFor
            \Let{$\hat{I}(X;Y|Z)$}{$\frac{1}{N}\sum\limits_{n=1}^{N}\hat{I}_n(X;Y|Z)$}
        \EndFunction
    \end{algorithmic}
\end{algorithm}

\begin{figure*}[t]
    \centering
    \includegraphics[keepaspectratio, width=\textwidth]{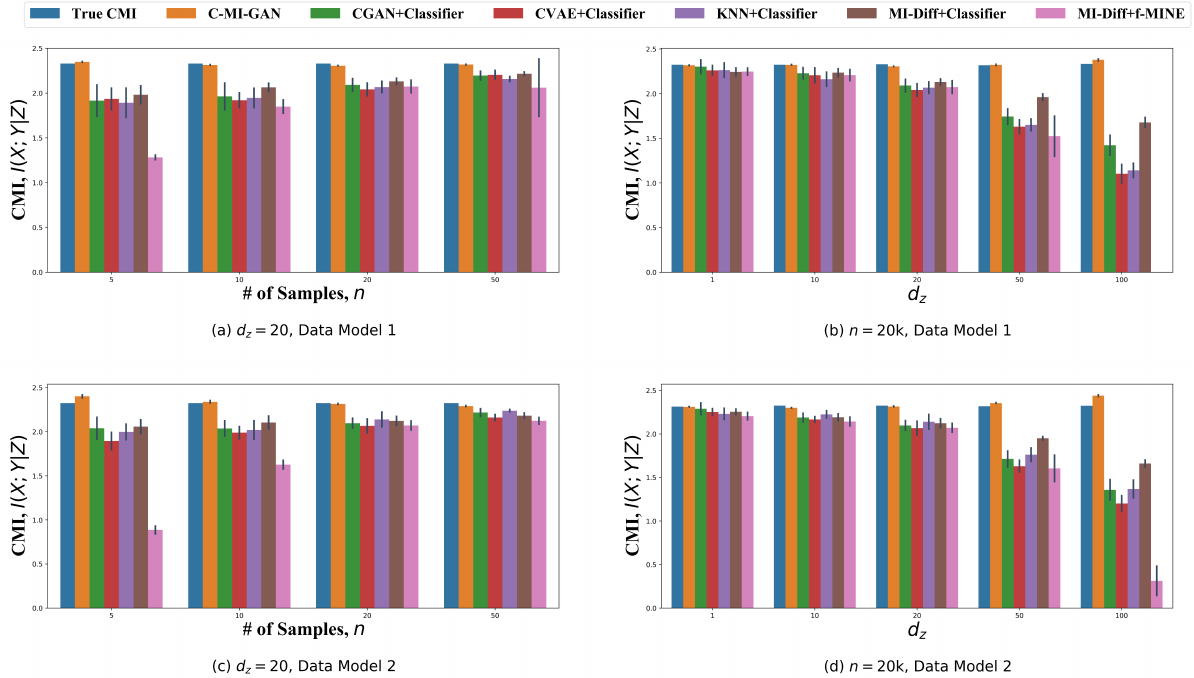}
    \caption{Performance of CMI estimators on the dataset generated using linear models. (a) Model 1: CMI estimates with fixed $d_z=20$ and variable sample size, (b) Model 1: CMI estimates with fixed sample size, $n=20$k and variable $d_z$, (c) Model 2: CMI estimates with fixed $d_z=20$ and variable sample size, (d) Model 2: CMI estimates with fixed number of samples, $n=20k$ and variable $d_z$. Average over 10 runs is plotted. Variation in the estimates are measured using standard deviation and are highlighted in the plots using the thin dark lines on top of the bar plots. The proposed method, C-MI-GAN provides closer estimate of the true CMI, while the state of the art estimators largely underestimate its value. C-MI-GAN outperform the state of the art methods in terms of variation in estimates as well. (Best viewed in color)}
    \label{fig:lin_cmi_comparison}
    \vspace{-5mm}
\end{figure*}{}

\section{EXPERIMENTAL RESULTS}

In this section we compare the CMI estimates, on different datasets, of our proposed method against the estimations of the state of the art CMI estimators such as f-MINE (\cite{belghazi}) and CCMI (\cite{mukherjee2019ccmi}). We design similar experiments on similar datasets as \cite{mukherjee2019ccmi} to demonstrate the efficacy of our proposed estimator: C-MI-GAN. 
Unlike the method proposed in this work, the existing methods rely on a separate generator for generating samples from the conditional distribution. Therefore, ``Generator''+``Divergence  estimator'' notation is used to denote the estimators used as baseline. For example, CVAE+f-MINE implies that Conditional VAE (\cite{NIPS2015_5775}) is used for generating samples from $P_{Y|Z}$ and f-MINE (\cite{belghazi}) is used for divergence estimation. MI difference based estimators are represented as MI-Diff.+``Divergence estimator''. For baseline models we have used the codes available in the repository of \cite{mukherjee2019ccmi}. Architecture for $R_\phi$ and $G_\theta$ and hyper-parameter settings for our proposed method are provided in the supplementary.

To illustrate C-MI-GAN's effectiveness in estimating CMI, we consider two synthetic and one real datasets: 
\begin{itemize}
    \item Synthetically generated datasets having linear dependency.
    \item Synthetically generated datasets having non-linear dependency.
    \item Air quality real dataset\footnote{\url{http://archive.ics.uci.edu/ml/datasets/Air+Quality}}\footnote{For a detailed description of the dataset please refer to the supplementary material.} (\cite{aq_data})
\end{itemize}{}
The most severe problem with the existing CMI estimators is that their performance drop significantly with increase in data dimension. To see how well the proposed estimator fares, compared to the existing estimators for high dimensional data, we vary the dimension, $d_z$ of the conditioning variable, $Z$ over a wide range, in both the datasets. For the non-linear dataset $d_x=d_y=1$, as found commonly in the literature on causal discovery and independence testing (\cite{mukherjee2019ccmi}, \cite{sen2017modelpowered}, \cite{10.5555/3020751.3020766}). To validate, how well the proposed estimator performs on a multi-dimensional $X$ and $Y$, we vary $d_x=d_y \in \{1, 5, 10, 15\}$ for the linear dataset. Besides, we consider datasets having as low as $5$k to as high as $50$k samples to understand the behaviour of the estimators as sample complexity varies.\par
\textbf{Ground Truth CMI:} For the datasets with linear dependency, ground truth CMI can be computed by numerical integration. However, to the best of our knowledge, there is no analytical formulation to compute the ground truth CMI for the synthetic datasets with non-linear dependency. As a workaround to this issue, as proposed by \cite{mukherjee2019ccmi}, we transform $Z$ as $U=A_{zy}Z$, where $A_{zy}$ is a random vector with entries drawn independently from $\mathcal{N}(0, 1)$ and then normalized to have unit norm. Following this transformation, $I(X;Y|Z)=I(X;Y|U)$. Since, $U$ has unity dimension $I(X;Y|U)$ can be estimated accurately using KSG given sufficient samples, as shown by \cite{gao2018demystifying} in the asymptotic analysis of KSG. Hence, a set of $50000$ samples is generated separately for each data-set to estimate $I(X;Y|U)$ and the estimated value is used as the ground truth for that data-set.\par
Next, as a practical application of CMI estimation, we test the null hypothesis of conditional independence on a synthetic dataset, and on real flow-cytometry data.
\begin{figure*}[t]
    \centering
    \includegraphics[keepaspectratio, width=\textwidth]{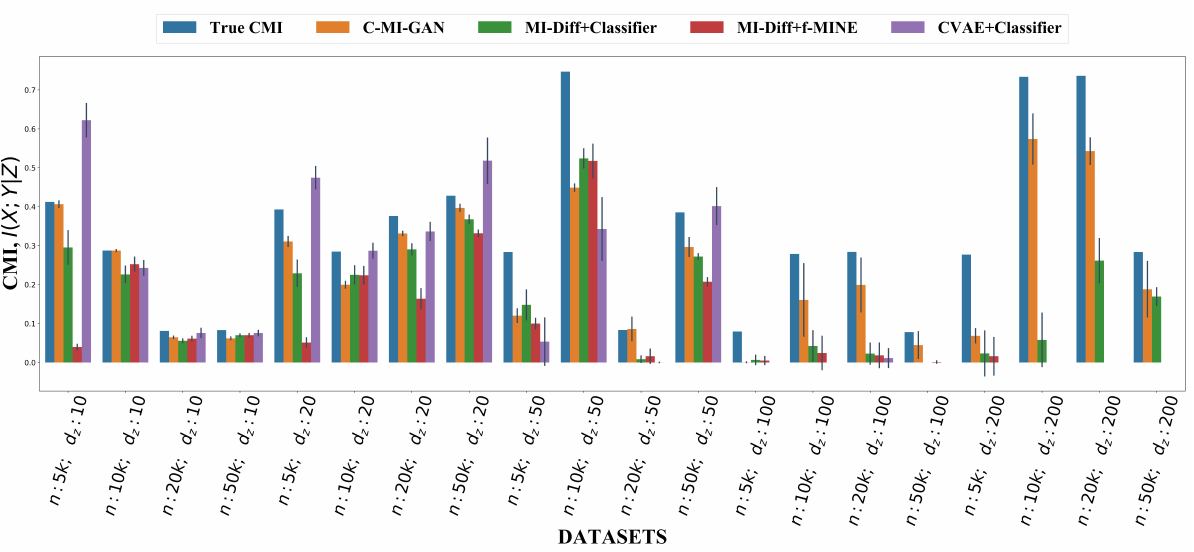}
    \caption{This figure compares the performance of the different CMI estimators on all the 20 non linear datasets. However, due to very poor performance of ``Generator''+``Classifier'' estimators, we plot the estimates of ``CVAE''+``Classifier'' only as a representative of that class of estimators. Estimated CMI, averaged over 10 runs is plotted. Standard deviation is indicated with the thin dark lines on top of the bar plot. Like in the linear case, the proposed C-MI-GAN outperforms the state of the art estimators in terms of both average CMI estimation and variation in estimation (Best viewed in color).}
    \label{fig:non_lin_cmi_comparison}
    \vspace{-5mm}
\end{figure*}{}
\subsection{CMI ESTIMATION}
\subsubsection{Dataset With Linear Dependence}\label{lin_data_models}
We consider the following data generative models.
\begin{itemize}
    \item \textbf{Model 1:} $X \sim \mathcal{N}(0, 1); Z \sim \mathcal{U}(-0.5, 0.5)^{d_z}; \epsilon \sim \mathcal{N}(Z_1, 0.01); Y \sim X+\epsilon$. \\
    In this model $X$ is sampled from standard normal distribution. Value of each dimension of $Z$ is drawn from a uniform distribution with support $[-0.5, 0.5]$. Finally $Y$ is obtained by perturbing $X$ with $\epsilon$, where $\epsilon$ comes from a Gaussian distribution having mean $Z_1$, the first dimension of $Z$ and variance $0.01$. Therefore, the dependence between $X$ and $Y$ is through the first dimension of the conditioning variable $Z$.
    \item \textbf{Model 2:} $X \sim \mathcal{N}(0, 1); Z \sim \mathcal{N}(0, 1)^{d_z}; U=w^TZ; ||w||_1=1;\epsilon \sim \mathcal{N}(U, 0.01); Y \sim X+\epsilon$\\
    Unlike in model 1, here the dimensions of $Z$ comes from standard normal distribution and the mean of $\epsilon$ is weighted average of all the dimensions of $Z$. The weight vector $w$ is constant for a particular dataset and varies across datasets generated by model 2.
    \item \textbf{Model 3:} $X \sim \mathcal{N}(0, 0.25)^{d_x}; Z^{(i)} \sim \mathcal{U}(-0.5, 0.5)^{d_z}; \epsilon \sim \mathcal{N}(Z_1, 0.25)^{d_z}; Y \sim X+\epsilon$. \\
    This model is very similar to Model 1, except for the fact $d_x = d_y = d_z \in \{5, 10, 15, 20\}$. Since, each dimension of the random variables are independent of each other, the ground truth CMI is estimated as sum of CMI of each dimension.
\end{itemize}{}
\vspace{-5mm}
For Models 1 and 2, to study the effect of sample size on the estimate we generate data by fixing $d_z=20$ and $n \in \{5000,10000,20000,50000\}$. Next, we fix $n=20000$ and $d_z \in \{1,10,20,50,100\}$ to observe the effect of dimension on the estimation.\par
Figure \ref{fig:lin_cmi_comparison} compares the average estimated CMI over $10$ runs for linear datasets. C-MI-GAN estimates are usually closer to the ground truth and exhibits less variation as compared to other estimators.\par
Next, we consider a multidimensional dataset generated using Model 3. We tabulate (refer to Table \ref{table:multi_dx_dy_dz_linear_cmi}) the average estimate of C-MI-GAN and the current state-of-the-art CCMI\footnotemark[5] (\cite{mukherjee2019ccmi}) over $10$ executions.
\begin{table}[ht]
    \begin{center}
    \caption{CMI Estimates on Multidimensional Linear Dataset Generated Using Model 3. All datasets have $d_x=d_y=d_z$ (Dim.). C-MI-GAN estimates are closer to ground truth and has less standard deviation.}
    \label{table:multi_dx_dy_dz_linear_cmi}
	\resizebox{\columnwidth}{!}{
        \begin{tabular}{c c c c }
            \hline
            Dim. & True CMI & CCMI & C-MI-GAN \\
            \hline
            % $5$  & $1.75$ & $1.61 \pm 4.45 \times 10^{-2}$ & $\boldsymbol{1.7 \pm 1.6 \times 10^{-5}}$   \\
            % $10$ & $3.48$ & $2.96 \pm 9.87 \times 10^{-2}$ & $\boldsymbol{3.34 \pm 1.49\times 10^{-4}}$  \\
            % $15$ & $5.22$ & $4.2 \pm 2.65 \times 10^{-1}$  & $\boldsymbol{5.03 \pm 6.57 \times 10^{-4}}$ \\
            % $20$ & $6.99$ & $4.8 \pm 4.71 \times 10^{-1}$  & $\boldsymbol{6.91 \pm 4.32 \times 10^{-2}}$ \\
            $5$  & $1.75$ & $1.61 \pm 4.45e-2$ & $\boldsymbol{1.7 \pm 1.6e-5}$   \\
            $10$ & $3.48$ & $2.96 \pm 9.87e-2$ & $\boldsymbol{3.34 \pm 1.49e-4}$  \\
            $15$ & $5.22$ & $4.2 \pm 2.65e-1$  & $\boldsymbol{5.03 \pm 6.57e-4}$ \\
            $20$ & $6.99$ & $4.8 \pm 4.71e-1$  & $\boldsymbol{6.91 \pm 4.32e-2}$ \\
            \hline \hline
        \end{tabular}
        }
        \vspace{-10mm}
    \end{center}
\end{table}
\subsubsection{Dataset With Non-Linear Dependence}\label{non-lin-data-model}
\textbf{Data generating model:}
$$Z \sim \mathcal{N}(\mathbbm{1}, I_{d_z})$$
$$X=f_1(\eta_1)$$
$$Y=f_2(A_{zy}Z+A_{xy}X+\eta_2)$$
Where, $f_1, f_2 \in \{ \cos(\cdot), \tanh(\cdot), \exp(-|\cdot|) \}$ and selected randomly; $\eta_1, \eta_2 \sim \mathcal{N}(0, 0.1)$. The elements of the random vector $A_{zy}$ are drawn independently from $\mathcal{N}(0, 1)$. The vector is then normalized to have unit norm. Since, $d_x=d_y=1$, $A_{xy}=2$ is a scalar.\par
To generate the data we consider all possible combinations of $n \in \{5000,10000,20000,50000\}$ and $d_z \in \{10,20,50,100,200\}$, and obtain a set of $20$ data-sets.\par
Figure \ref{fig:non_lin_cmi_comparison} plots the average estimate of different estimators over $10$ runs for all $20$ data-sets. Error bar (standard deviation) is plotted as well on top of the estimation. Only MI-Diff.+Classifier among the existing estimators provides reasonable estimate when $d_z$ is high, while the proposed method tracks the true CMI more closely.
\subsubsection{Real Data: Air Quality Data Set}
Finally, we estimate CMI on a real dataset of air quality (\cite{aq_data}) using C-MI-GAN. We tabulate the estimates obtained in Table \ref{table:aq_cmi}. We consider the causal graph discovered by \cite{rungeCMIkNN} representing the dependencies between three pollutants and three meteorological factors (refer to supplementary material) as the ground truth. The graph captures the strength of these dependencies. However, the ground truth CMI being unknown, we could validate only the ordering of the estimated CMI. As can be seen from the table, our estimates are in coherence with the findings of \cite{rungeCMIkNN} and the estimates of CCMI\footnote{MI-diff + Classifier} (\cite{mukherjee2019ccmi}) and KSG (\cite{kraskov2004}).For example, $I(\text{CO}; \text{C\textsubscript{6}H\textsubscript{6}}| \text{T})$ is higher as compared to $I(\text{C\textsubscript{6}H\textsubscript{6}}; \text{RH} | \text{T})$ , indicating a relatively stronger conditional dependency between the former pair. This is in compliance to the ground truth causal graph (see supplementary material).
\begin{table}[hb]
    \vspace{-1mm}
    \caption{CMI Estimation: Air Quality Dataset}
    \vspace{-3mm}
    \label{table:aq_cmi}
    \begin{center}
    \resizebox{\columnwidth}{!}{
    \begin{tabular}{c c c c c c}
        \hline
        $X$                                  & $Y$                           & $Z$                                  &   CCMI   &   KSG     &    C-MI-GAN     \\
        \hline
        CO                                     &  C\textsubscript{6}H\textsubscript{6} &  T                                           & $0.61$   &   $0.65$   &    $0.66$      \\
        CO                                     &  C\textsubscript{6}H\textsubscript{6} &  NO\textsubscript{2}                         & $0.33$   &   $0.37$   &    $0.37$      \\
        CO                                     &  C\textsubscript{6}H\textsubscript{6} &  RH                                          & $0.56$   &   $0.60$   &    $0.59$      \\
        CO                                     &  C\textsubscript{6}H\textsubscript{6} &  AH                                          & $0.58$   &   $0.63$   &    $0.62$      \\
        NO\textsubscript{2}                    & C\textsubscript{6}H\textsubscript{6}  &  AH                                          & $0.40$   &   $0.45$   &    $0.45$      \\
        NO\textsubscript{2}                    & C\textsubscript{6}H\textsubscript{6}  &  T                                           & $0.38$   &   $0.43$   &    $0.44$      \\
        \hline
        NO\textsubscript{2}                    & CO                                    &  C\textsubscript{6}H\textsubscript{6}        & $0.06$   &   $0.11$   &    $0.12$      \\
        NO\textsubscript{2}                    & RH                                    &  C\textsubscript{6}H\textsubscript{6}        & $0.01$   &   $0.05$   &    $0.05$      \\
        C\textsubscript{6}H\textsubscript{6}   & AH                                    &  T                                           &  $0.01$  &   $0.07$   &    $0.07$      \\
        C\textsubscript{6}H\textsubscript{6}   & RH                                    &  T                                           &  $0.02$  &   $0.06$   &    $0.06$      \\
        CO                                     & AH                                    & C\textsubscript{6}H\textsubscript{6}         & $0.03$   &   $0.06$   &    $0.07$      \\
        RH                                     & CO                                    & C\textsubscript{6}H\textsubscript{6}         & $0.04$   &   $0.09$   &    $0.09$      \\
        \hline \hline \\
    \end{tabular}
    }
    \end{center}
\end{table}
\vspace{-9mm}
\subsection{APPLICATION: CONDITIONAL INDEPENDENCE TESTING (CIT)}
\subsubsection{Synthetic Dataset}
To evaluate the proposed CMI estimator on an application, we consider testing the null hypothesis of conditional independence, used widely in conventional literature (\cite{sen2017modelpowered, mukherjee2019ccmi}). The objective here is to decide, whether $X$ and $Y$ are independent given $Z$ when we have access to samples from the joint distribution. Formally, given samples from the distributions $P(x,y,z)$ and $Q(x,y,z)$ where $Q(x,y,z)=P(x,z)P(y|z)$ , we have to test our estimators on the hypothesis testing framework given by the null, $H_0: X \perp Y|Z$  and the alternative, $H_1: X \not\perp Y|Z$.\par
The conditional independence test setting will be used to test our estimator based on the fact that $X \perp Y|Z \iff I(X;Y|Z)=0$. A simple rule can thus be established: reject the null hypothesis if $I(X;Y|Z)$ is greater than some threshold (to allow some tolerance in the estimation) and accept it otherwise.\par 
CIT can be cast as binary classification problem where samples belong to either class-CI or class-CD. Therefore, area under ROC curve (AuROC) is a good metric to compare the performance of different algorithms. Therefore, we consider the AuROC scores of different models for performance comparison.\par
\begin{figure}[!ht]
    \centering
    \includegraphics[clip, trim=0.3cm 0.3cm 0.3cm 0.3cm, keepaspectratio, width=\columnwidth]{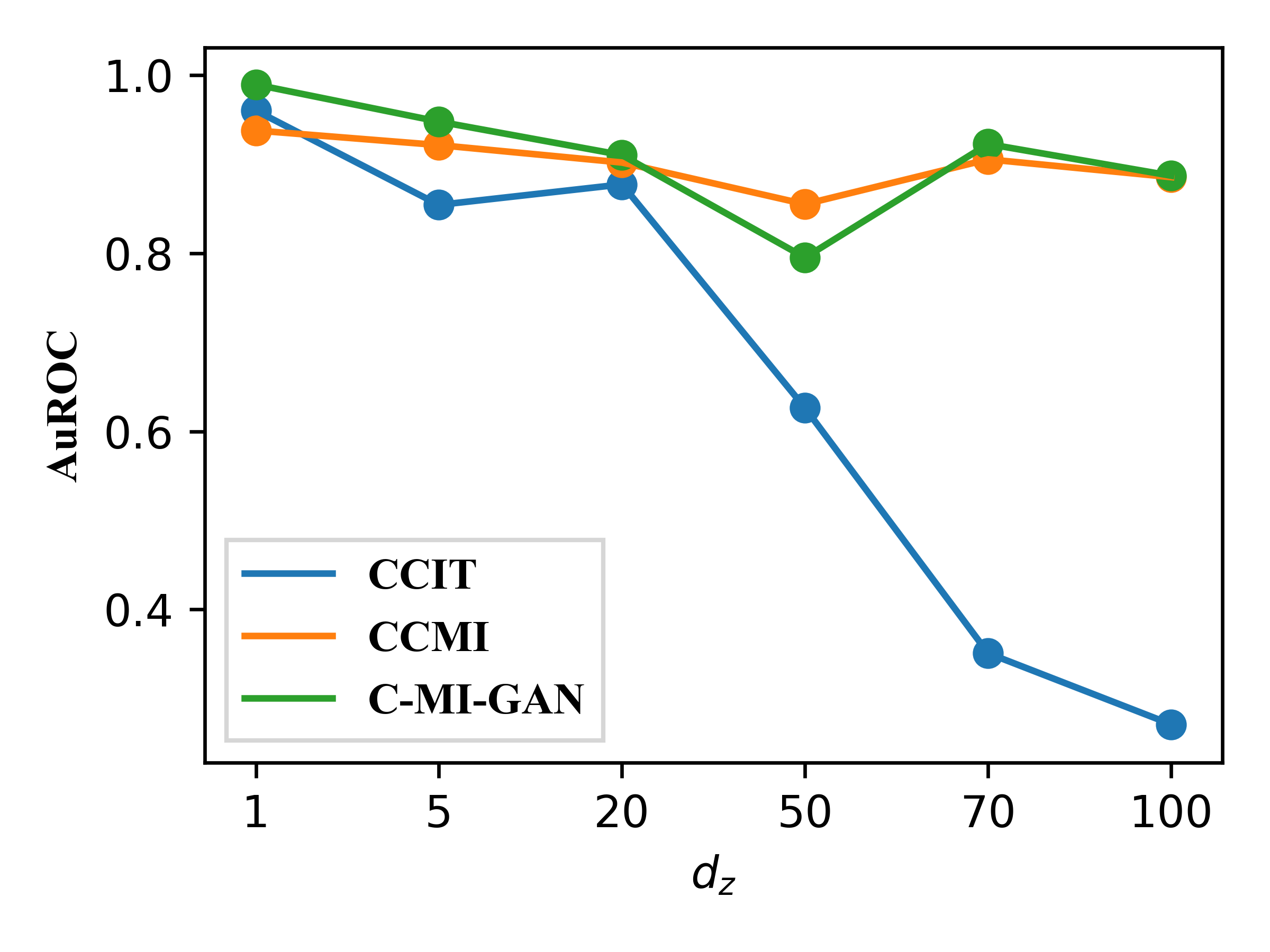}
    \caption{Performace of CCIT degrades with $d_z$. CCMI and C-MI-GAN are comparable across all $d_z$. (Best viewed in color). }
    \label{fig:syn_cit_auroc}
    \vspace{-3mm}
\end{figure}{}

Synthetic data is generated using the post non-linear noise model as used by \cite{sen2017modelpowered} and \cite{mukherjee2019ccmi}. The data generation model is as follows:
$$Z \sim \mathcal{N}(\mathbbm{1}, I_{d_z})$$
$$X=\cos(a_xZ+\eta_1)$$
$$Y=\begin{cases}
      \cos(b_yZ+\eta_2), & \text{if}\ X \perp Y |Z \\
      \cos(cX+b_yZ+\eta_2), & \text{otherwise}
    \end{cases}$$
$\eta_1, \eta_2 \sim \mathcal{N}(0, 0.25)$; $a_x, b_y \sim \mathcal{U}(0, 1)^{d_z}$ and normalized such that $||a_x||_2 = ||b_y||_2 = 1$; $c \sim \mathcal{U}(0, 2)$. As before the model parameters $a_x, b_y,\text{ and } c$ are kept constant for a particular dataset but varies across datasets. $d_x=d_y=1$ and $d_z \in \{1, 5, 20, 50, 70, 100\}$. $100$ datasets consisting of $50$ conditionally independent and $50$ conditionally dependent datasets are generated for each $d_z$. Sample size of each dataset is fixed as $n=5000$.\par
Figure \ref{fig:syn_cit_auroc} compares the performance of C-MI-GAN with CCIT (\cite{sen2017modelpowered}) and CCMI (\cite{mukherjee2019ccmi}). Performance of CCIT degrades rapidly as $d_z$ increases. Performance of C-MI-GAN and CCMI remains comparable for all $d_z$. Performance of C-MI-GAN remains undeterred with increasing dimensions.
\subsubsection{Flow-Cytometry: Real Data}
To test the efficacy of our proposed method in conditional independence testing on real data, we have used Flow cytometry dataset introduced by \cite{Sachs523}. This dataset quantifies the availability of 11 biochemical compounds in single human immune system cells under different test conditions. Please refer to the supplementary material for the consensus network, which serves the purpose of ground truth. It depicts the causal relations between the 11 biochemical compounds.\par 
The underlying concept for generating the CI and CD datasets is similar to that used in \cite{sen2017modelpowered} and \cite{mukherjee2019ccmi}.  A node $X$ is conditionally independent of any other unconnected node $Y$ given its Markov blanket i.e. its parents, children and co-parents of children. So given $Z$ consisting of the parents, children and co-parents of children of $X$, $X$ is conditionally independent of any other node $Y$. Also, if a direct edge exists between $X$ and $Y$, then given any $Z$, $X$ is not conditionally independent of $Y$. We have used this philosophy to create $70$ CI and $54$ CD datasets.\par
\cite{Sachs523} and \cite{10.5555/3023638.3023682} used a subset of $8$ of the available $14$ original flow cytometry datasets in their experiments to come up with Bayesian networks representing the underlying causal structure. We also used those 8 datasets in our experiments which had a combined total of around $7000$ samples. The dimension of $Z$ varies in the range $3$ to $8$.\par
Figure \ref{fig:flow_cyto_roc} compares the AuROC score of C-MI-GAN against the scores of CCMI and CCIT. C-MI-GAN retains its superior performance when compared against CCMI and CCIT. Surprisingly, CCIT outperforms CCMI contradicting the result presented by \cite{mukherjee2019ccmi}. This discrepancy might be due to limited capacity of their model architecture. We have created a larger dataset consisting of around $7000$ samples. Whereas, the numbers reported by \cite{mukherjee2019ccmi} are based on a smaller subset consisting of 853 data points. A larger network might improve the performance of CCMI. \par
\begin{figure}[!ht]
    \centering
    \includegraphics[clip, trim=0.5cm 0.5cm 0.5cm 0.5cm, keepaspectratio, width=\columnwidth]{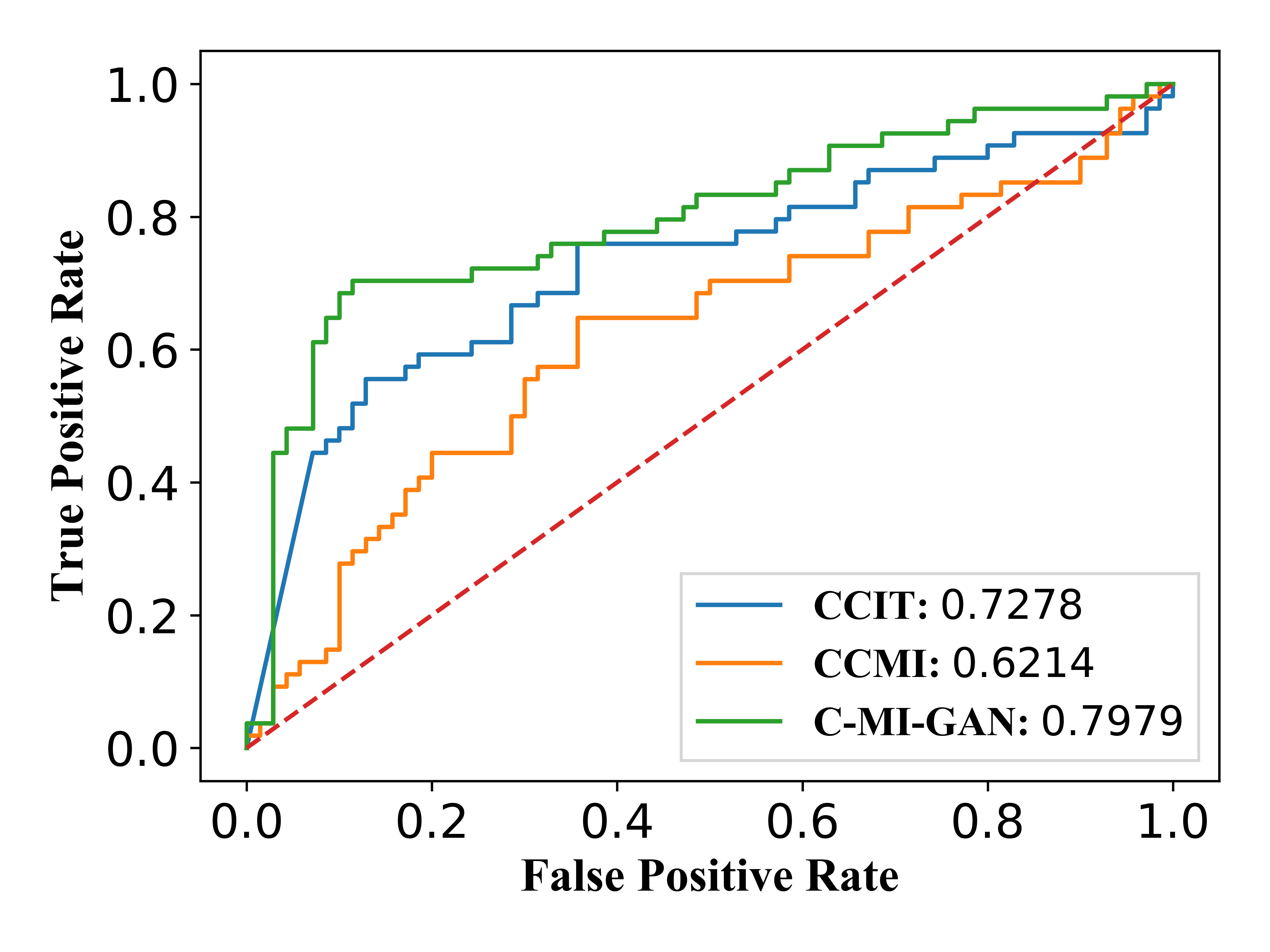}
    \caption{AuROC Curves: Flow-Cytometry  Data-set. CCIT obtains a mean AuROC score of 0.728, CCMI obtains a mean of 0.62 while C-MI-GAN outperforms both of them with a mean AuROC score of 0.798  (Best viewed in color).}
    \label{fig:flow_cyto_roc}
    \vspace{-5mm}
\end{figure}{}
Although \cite{shah2018hardness} argues that domain knowledge is necessary to select an appropriate conditional independence test for a particular dataset, C-MI-GAN outperforms both CCIT and CCMI (as measured using AuROC) consistently across all the experiments. This might be because, in the min-max formulation, the regression network, $R_\phi$ and the generator network, $G_\theta$ are trained jointly using an adversarial scheme. As a result, the two networks together cancel the bias present in each other resulting in performance boost in high sample regime (a regime where GANs excel).

\section{DISCUSSION AND CONCLUSION}
In this work, we propose a novel CMI estimator, C-MI-GAN. This estimator is based on the formulation of CMI as a min-max objective that can be optimized using joint training. We refrain from estimating two separate MI terms that could have unequal bias present. As opposed to separately training a conditional sampler and a divergence estimator, which may be sup-optimal, our joint training incorporates both steps into a single training procedure. We find that the estimator obtains improved estimates over a range of linear and non-linear datasets, across a wide range of dimension of the conditioning variable and sample size. Finally, we achieve performance boost in CI testing on simulated and real datasets using our improved estimator.

%% file: ack.tex
\section*{Acknowledgement}
We thank IIT Delhi HPC facility\footnote{\emph{http://supercomputing.iitd.ac.in}} for computational resources. Sreeram Kannan is supported by NSF awards 1651236 and 1703403 and NIH grant 5R01HG008164. Himanshu Asnani acknowledges the support of Department of Atomic Energy, Government of India, under project no. 12-R\&D-TFR-5.01-0500. Any opinions, findings, conclusions or recommendations expressed in this paper are those of the authors and do not necessarily reflect the views or official policies, either expressed or implied, of the funding agencies.

%% file: supp.tex
\section{MODEL ARCHITECTURES}
In this section we provide the architectural details of C-MI-GAN used for CMI estimation and conditional independance testing.
\begin{table}[ht]
    \caption{ Hyperparameters: CMI Estimation}
    \label{table:hyperparameter cmi est}
    \begin{center}
	\resizebox{\columnwidth}{!}{
        \begin{tabular}{c c c }
            \hline
            Hyperparameters & $R_\phi$ & $G_\theta$\\
            \hline
            \# Hidden layers & 2 & 2\\
            %\hline%
            Hidden Units & [128, 32] & [256, 64]\\
            %\hline%
            Activation: Hidden Layers & ReLu & ReLu\\
            Activation: Final Layer & Identity & Identity\\
            %hline%
            Batch Size & 4096 & 4096\\
            Initial Learning rate & $5\times10^{-5}$ & $5\times10^{-5}$\\
            Optimizer & RMSprop & RMSprop\\
            % Regularizer & L2 (0.0001) & L2(0.0001)\\
            \# Training Steps & $30$k & $30$k\\
            % $d_z$ to Noise dim. ratio & - & 4\\
            % Noise distribution & - & $\mathcal{N}(0, I_{\frac{d_z}{4}})$\\
            Learning rate scheduling interval & $1$k & $1$k\\
            Learning rate decay factor & 10 & 10\\
            
            \hline \hline \\
        \end{tabular}
        }
    \end{center}
\end{table}
\begin{table}[ht]
    \begin{center}
    \caption{ Hyperparameters: CIT Application}
    \label{table:Hyperparameter cit application}
	\resizebox{\columnwidth}{!}{
        \begin{tabular}{c c c }
            \hline
            Hyperparameters & $R_\phi$ & $G_\theta$\\
            \hline
            \# Hidden layers & $3$ & $3$\\
            %\hline%
            Hidden Units & $[128, 32, 8]$ & $[128, 64, 16]$\\
            %\hline%
            Activation: Hidden Layers & ReLu & ReLu\\
            Activation: Final Layer & Identity & Identity\\
            %hline%
            Batch Size & $4096$ & $4096$\\
            Initial Learning rate & $0.001$ & $0.001$\\
            Optimizer & RMSprop & RMSprop\\
            % Regularizer & L2 (0.0001) & L2(0.0001)\\
            \# Training Steps & 10k & 10k\\
            Learning rate scheduling interval & $1$k & $1$k\\
            Learning rate decay factor & $10$ & $10$\\
            
            \hline \hline \\
        \end{tabular}
        }
    \end{center}
\end{table}

\section{MI-GAN}
Ideas presented in Sec. \ref{sec:cmi_min_max} in the main paper can be applied to estimating mutual information as well.

\begin{equation}
    \begin{split}
        I(X;Y) &= D_{KL}(P_{XY}||P_XP_Y) \\
        &= D_{KL}(P_{XY}||P_{X}Q_{Y}) - D_{KL}(P_Y || Q_Y) \\
        &\le D_{KL}(P_{XY}||P_{X}Q_{Y})
    \end{split}
    \label{mi_upper_bound}
\end{equation}
In equation \ref{mi_upper_bound} the equality is achieved when $Q_Y = P_Y$ and it may be expressed as
\begin{equation}
    \begin{split}
        I(X;Y) = \mathop{\text{inf}}_{Q_Y} D_{KL}(P_{XY}||P_{X}Q_{Y})
    \end{split}
    \label{mi_inf_formulation}
\end{equation}
Equation \ref{mi_inf_formulation} coupled with the Donsker-Varadhan bound leads to a min-max optimization for MI estimation as below.
\begin{equation}
    \begin{split}
        I(X;Y) &= \mathop{\text{inf}}_{Q_Y}~\mathop{\text{sup}}_{R \in \mathcal{R}}\bigg(\mathop{\mathbb{E}}_{s\sim P_{XY}}[R(s)] -\\ 
        &\qquad \log\big(\mathop{\mathbb{E}}_{s\sim P_XQ_Y}[e^{R(s)}]\big)\bigg)
    \end{split}
    \label{mi_min_max_formulation}
\end{equation}

Here we consider a simple setting of two correlated gaussian random variables in $2n$ dimensions, $ (X,Y) \sim \mathcal N \left (\overrightarrow{0}_{2n\times 1}, \Sigma=\begin{pmatrix}
I_{n\times n} & \rho I_{n\times n} \\ 
\rho I_{n\times n} & I_{n\times n}
\end{pmatrix} \right )$ and estimate $I(X;Y)$ using the min-max formulation as mentioned in equation (\ref{mi_min_max_formulation}). Next, we compare the results with the existing estimators. Figures \ref{fig:gaussian_dep_dx_1} and \ref{fig:gaussian_dep_dx_10} plots the estimated mi using MI-GAN, Classifier MI and f-MINE.

\begin{figure}[!ht]
\centering
\begin{subfigure}{\columnwidth}
   \centering
    \includegraphics[clip, trim=0.5cm 0.5cm 0.5cm 0.5cm, keepaspectratio, width=\linewidth]{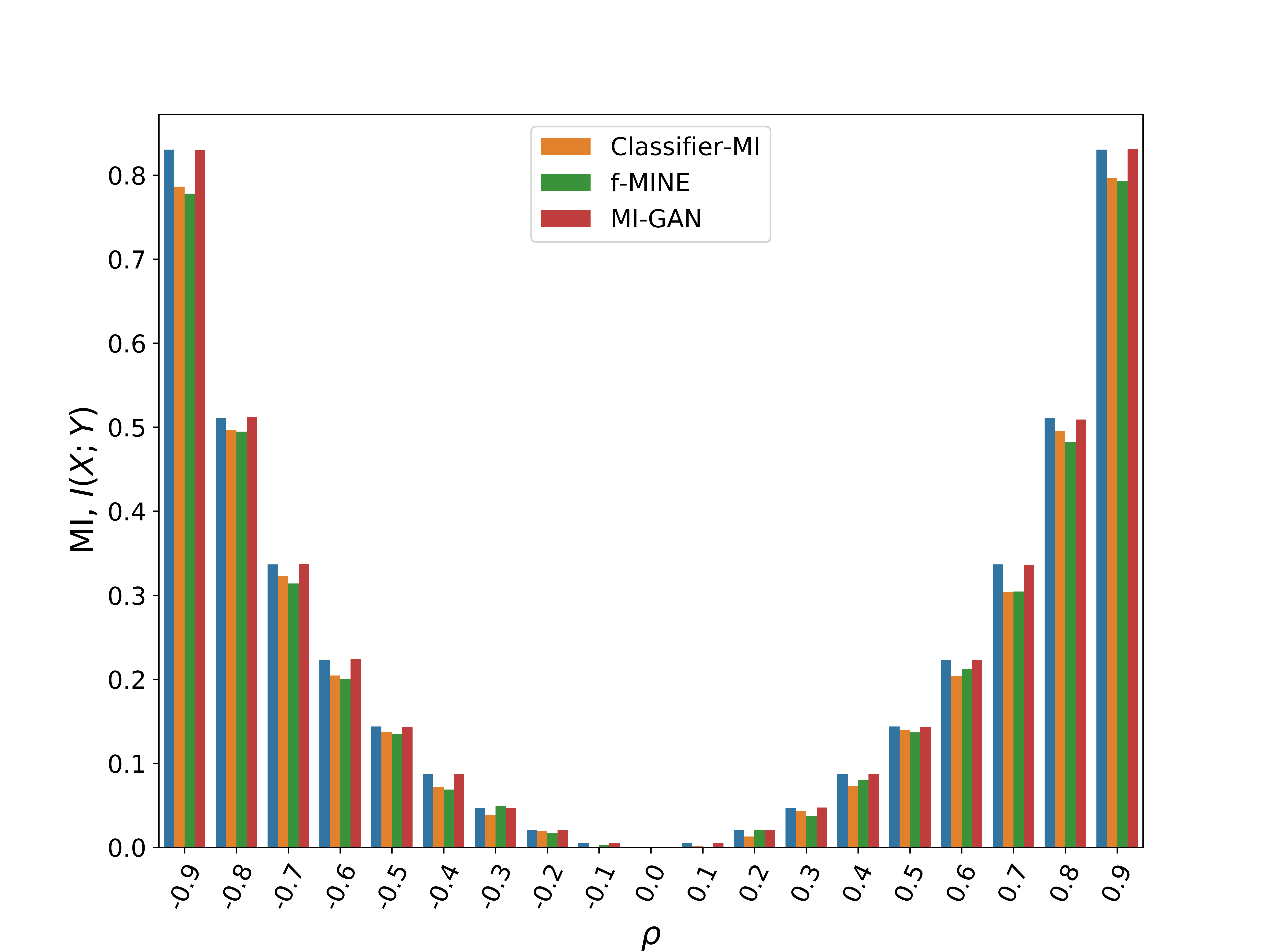}
    \caption{$d_X = d_y = 1$}
    \label{fig:gaussian_dep_dx_1}
\end{subfigure}

\begin{subfigure}{\columnwidth}
   \centering
    \includegraphics[clip, trim=0.5cm 0.5cm 0.5cm 0.5cm, keepaspectratio, width=\linewidth]{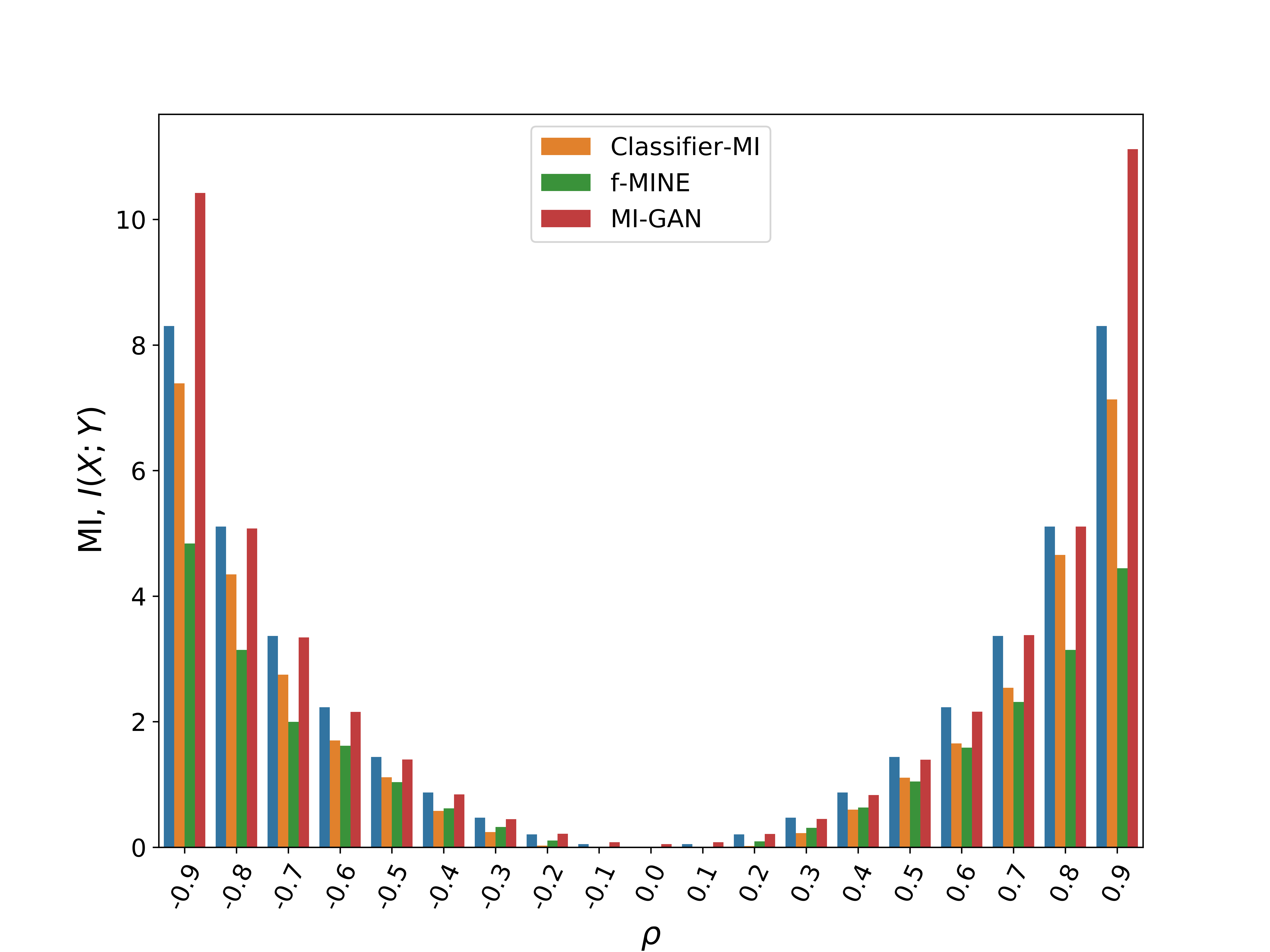}
    \caption{$d_X = d_y = 10$}
    \label{fig:gaussian_dep_dx_10}
\end{subfigure}

\caption{ Mutual Information Estimation of Correlated Gaussians: In this setting, $X$ and $Y$ have independent co-ordinates, with $(X_i,Y_i) ~ \forall i$ being correlated Gaussians with correlation coefficient $\rho$. $I^{*}(X;Y) =-\frac{1}{2}d_x\log(1-\rho^2)$ (Best viewed in color).}
\end{figure}

\begin{figure*}[t]
    \centering
    \includegraphics[keepaspectratio, width=\textwidth]{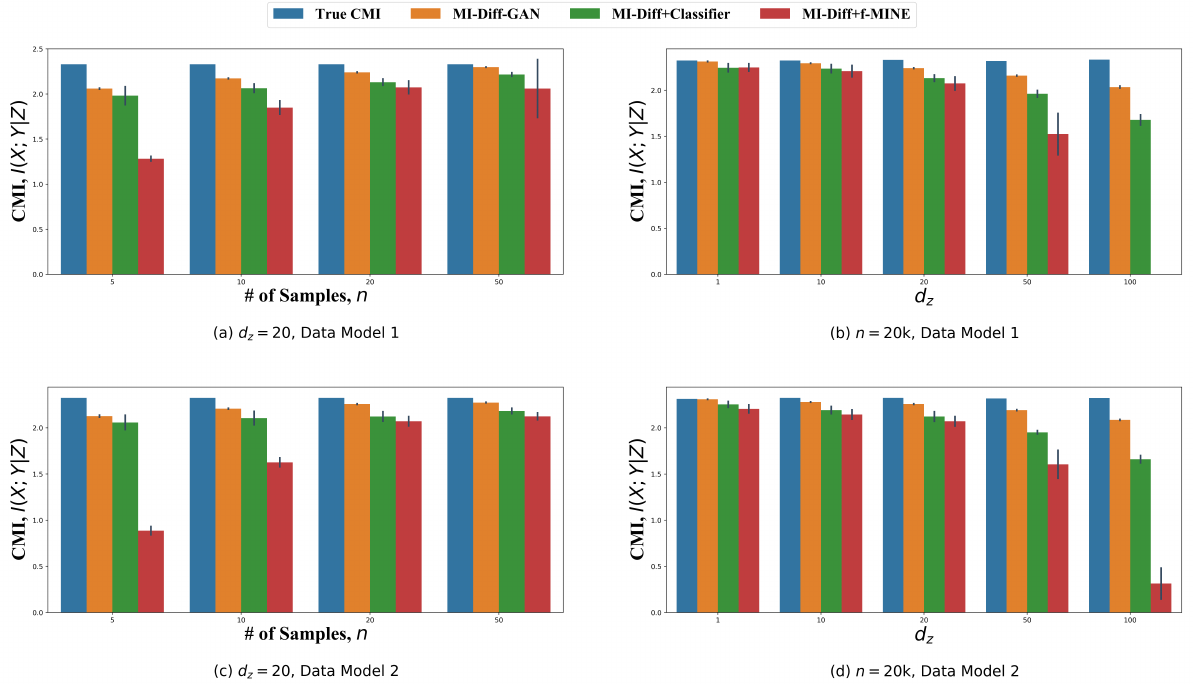}
    \caption{Performance of MI-Diff-GAN on the data generated using linear models proposed in Sec. \ref{lin_data_models} of the main paper. (a) Model 1: CMI estimates with fixed $d_z=20$ and variable sample size, (b) Model 1: CMI estimates with fixed sample size, $n=20$k and variable $d_z$, (c) Model 2: CMI estimates with fixed $d_z=20$ and variable sample size, (d) Model 2: CMI estimates with fixed number of samples, $n=20k$ and variable $d_z$. Here, we compare our results with other MI-Diff based estimators of CMI such as MI-Diff. + Classifier and MI-Diff. + f-MINE. Average over 10 runs is plotted. Variation in the estimates are measured using standard deviation and are highlighted in the plots using the thin dark lines on top of the bar plots. MI-Diff-GAN  outperforms the state of the art methods in terms of accuracy in cmi estimation and variation in estimates. (Best viewed in color)}
    \label{fig:sup_lin_cmi_comparison}
\end{figure*}{}

\begin{figure*}[t]
    \centering
    \includegraphics[keepaspectratio, width=\textwidth]{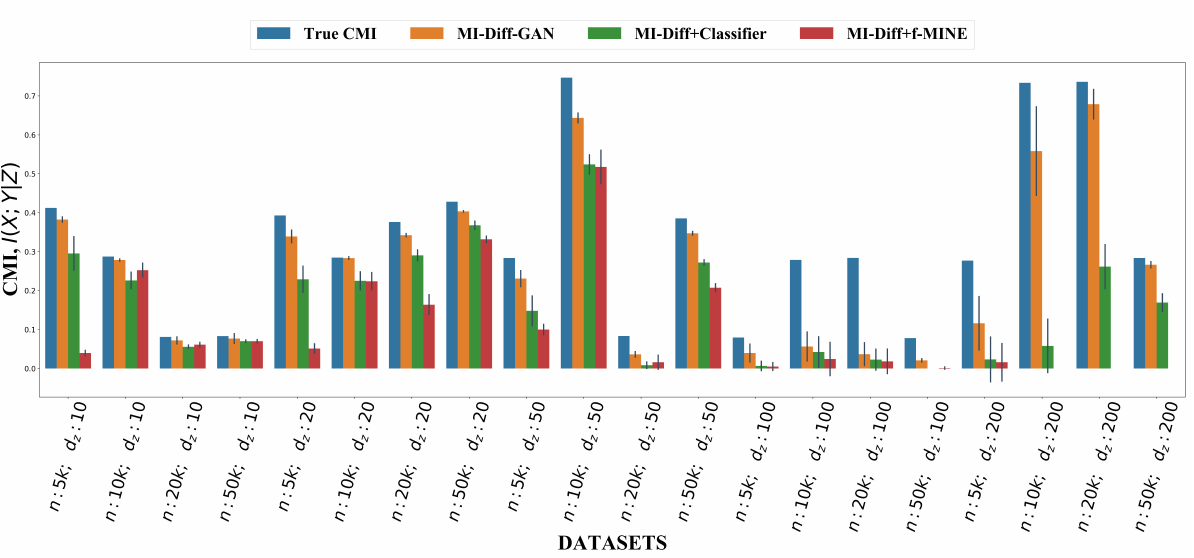}
    \caption{Performance of MI-Diff-GAN on the data generated using non linear models proposed in Sec. \ref{non-lin-data-model} of the main paper. Again we compare against MI-Diff based CMI estimators such as MI-Diff. + Classifier and MI-Diff. + f-MINE. Average over $10$ runs is plotted. Variation in the estimates are measured using standard deviation and are highlighted in the plots using the thin dark lines on top of the bar plots. MI-Diff-GAN performs better in terms of both estimation accuracy and variation in estimates. (Best viewed in color)}
    \label{fig:sup_non_lin_cmi_comparison}
\end{figure*}{}

\subsection{DIFFERENCE BASED MIN-MAX FORMULATION FOR CMI}\label{mi_diff_gan}
Conditional mutual information can be expressed as difference of two mutual information as below:
\begin{equation}
    I(X;Y|Z) = I(X;YZ) - I(X;Z)
    \label{cmi_diff_formula}
\end{equation}
Equation \ref{mi_min_max_formulation} and equation \ref{cmi_diff_formula} together leads to another min-max formulation of $I(X;Y|Z)$ requiring one generator and two regression networks. 
\begin{equation}
    \begin{split}
        I(X;Y|Z) &= I(X;YZ) - I(X;Z)\\
        &= \mathop{\text{inf}}_{Q_X}\bigg(\mathop{\text{sup}}_{R_1 \in \mathcal{R}}\Big(\mathop{\mathbb{E}}_{s\sim P_{XYZ}}[R_1(s)] \\
        &\qquad - \log\big(\mathop{\mathbb{E}}_{s\sim Q_{X}P_{YZ}}[e^{R_1(s)}]\big)\Big) \\
        &\qquad -  \mathop{\text{sup}}_{R_2 \in \mathcal{R}}\Big(\mathop{\mathbb{E}}_{s\sim P_{XZ}}[R_2(s)] \\
        &\qquad - \log\big(\mathop{\mathbb{E}}_{s\sim Q_{X}P_{Z}}[e^{R_2(s)}]\big)\Big)\bigg)
    \end{split}
    \label{cmi_mi_diff_min_max_formulation}
\end{equation}

We call the estimator based on difference based min-max formulation for CMI (equation \ref{cmi_mi_diff_min_max_formulation}) as MI-Diff.-GAN. We evaluate the model on the synthetic datasets proposed in Sections \ref{lin_data_models} and \ref{non-lin-data-model} of the main paper. In Figures \ref{fig:sup_lin_cmi_comparison} and \ref{fig:sup_non_lin_cmi_comparison}, we compare the results of MI-Diff.-GAN with the existing difference based CMI estimators. However a detailed comparison of MI-Diff.-GAN with all other state of the art estimators has been tabulated in Tables \ref{table:cmi_lin_model1_comparison}, \ref{table:cmi_lin_model2_comparison} and \ref{table:cmi_nonlin_comparison}, which comprise of the average estimated CMI and the variance in the estimation, calculated over $10$ runs, corresponding to each of the estimators discussed in this work. Moreover, the RMSE calculated between the true cmi and the average estimated cmi values for each of the estimators,  have been highlighted in Table \ref{table:cmi_rmse_estimators}. It can serve as a naive performance metric for the estimators. \par
From the results obtained from the various experimental setups discussed in this work , it is apparent that MI-Diff.-GAN and C-MI-GAN are at par with each other and superior to the  existing CMI estimators in terms of accuracy in CMI estimation and variation in estimates.
\section{AIR QUALITY DATASET}
Air Quality Dataset is a time series data with $9358$ instances of hourly averaged responses from an array of 5 metal oxide chemical sensors embedded in an Air Quality Chemical Multisensor Device, which was located on the field in a significantly polluted area, at road level, within an Italian city. Data were recorded over a span of $1$ year starting from March 2004 to February 2005. In this work, we consider different combinations of hourly averaged concentrations for carbon monoxide (CO), benzene (C6H6), nitrogen dioxide (NO2), as well as temperature(T), relative humidity (RH) and absolute humidity(AH) as random variable and estimate CMI. In our experiments, $X$, $Y$, and $Z$ are always uni-dimensional. Before estimating CMI, we shuffle the samples to remove the temporal attributes and normalize the data. For the estimates of the proposed method, please refer to Table \ref{table:aq_cmi} in the main paper. As can be seen from the table, our estimates are in coherence with the findings of \cite{rungeCMIkNN} and the estimates of CCMI (\cite{mukherjee2019ccmi}) and KSG (\cite{kraskov2004}). Please note that the air quality dataset being a time series data, is not i.i.d data. Thus, there exists correlation between variables.
\begin{figure}[H]
    \centering
    \includegraphics[keepaspectratio, width=\columnwidth]{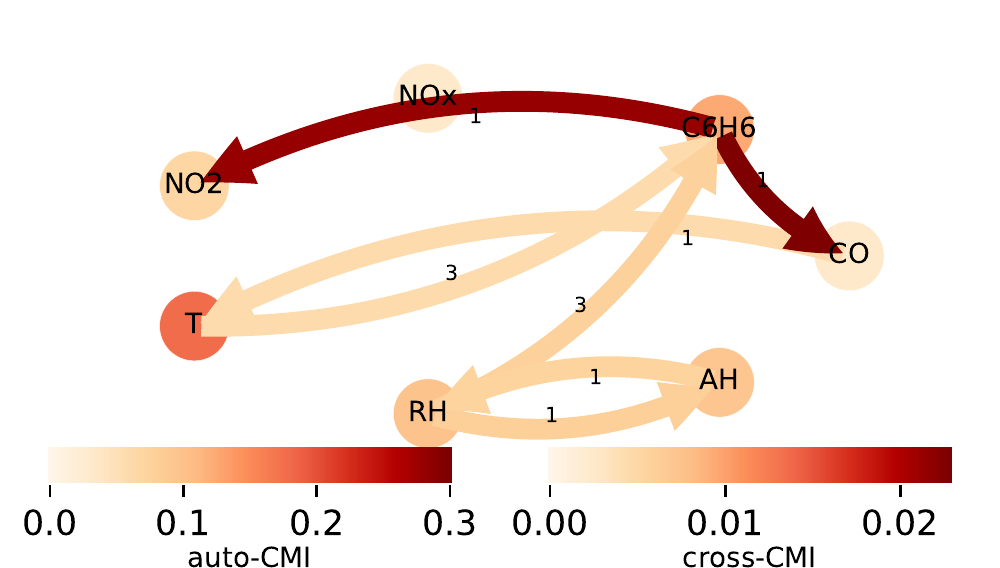}
    \caption{Causal graph as discovered by \cite{rungeCMIkNN} for air quality dataset (\cite{aq_data}).The edge color gives the strength of the CMI between different nodes with the edge labels denoting the time lag in hours. This figure is taken as it is from \cite{rungeCMIkNN} and included here for quick reference.}
    \label{fig:airquality_graph}
\end{figure}

\section{FLOW CYTOMETRY DATASET}
The Flow Cytometry dataset was originally developed by \cite{Sachs523} in order to study and derive causal influences in cellular signaling networks. Multiple phosphorylated protein and phospholipid components were simultaneously measured in thousands of individual primary human immune system cells. In order to observe the ordering of connections between these components constituting molecular pathways, these cells were perturbed with molecular interventions. Data pertaining to 14 such types of interventions on the molecular components are available ( Table 1 in \cite{Sachs523}), out of which we chose a subset of 8 data sets for our experiments.
Figure \ref{fig:flow_cytometry_graph} consists of the consensus graph that was used in \cite{Sachs523}, \cite{sen2017modelpowered} and \cite{10.5555/3023638.3023682}  as a ground truth network for causal discovery and conditional independence testing applications. We also considered this as the ground truth for our experiments in CI testing.
\begin{figure}[H]
    \centering
    \includegraphics[keepaspectratio, width=0.7\linewidth]{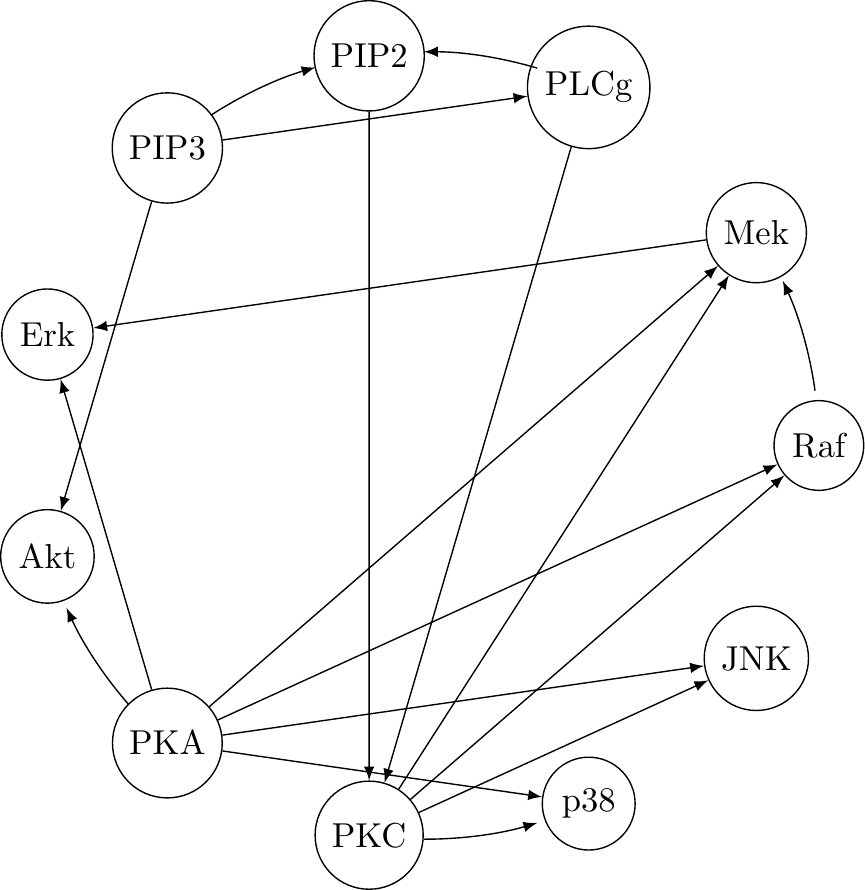}
    \caption{Consensus network, according to \cite{Sachs523}.}
    \label{fig:flow_cytometry_graph}
\end{figure}{}

\begin{table*}[ht]
    \begin{center}
    \caption{ Results Of CMI Estimation On Linear Model I}
    \label{table:cmi_lin_model1_comparison}
	\resizebox{\linewidth}{!}{
        \begin{tabular}{c c c c c c c c c c c c c c c c c}
            \hline
            \multirow{2}{*}{$d_z$} & \multirow{2}{*}{$n$} & \multicolumn{2}{c}{CGAN+Classifier} &  \multicolumn{2}{c}{CVAE+Classifier} &  \multicolumn{2}{c}{KNN+Classifier} &  \multicolumn{2}{c}{MI-Diff.+Classifier} &  \multicolumn{2}{c}{MI-Diff.+f-MINE} &  \multicolumn{2}{c}{C-MI-GAN} &  \multicolumn{2}{c}{MI-Diff-GAN} & True CMI\\
            \cline{3-16}
            & & Avg. & Std. &  Avg. & Std. &  Avg. & Std. &  Avg. & Std. &  Avg. & Std. &  Avg. & Std. & Avg. & Std. &\\
            \hline
            20	& 5	& 1.916	& 0.184	& 1.936 & 0.126 & 1.892 & 0.171 &	1.982 &	0.105 &	1.283 &	0.027 &	2.348 &	0.004 &	2.06 & 0.007 & 2.329903976\\
            %\hline%
            20 & 10 & 1.963	& 0.157 & 1.92 &	0.086 &	1.947 &	0.113 &	2.064 &	0.049 &	1.85 &	0.078 &	2.314 &	0.004 & 2.173 &	0.004 & 2.329903976\\
            %\hline%
            20 & 20	& 2.091	& 0.074 & 2.041 &	0.075 &	2.068 &	0.066 &	2.131 &	0.036 &	2.074 &	0.075 &	2.306 &	0.003 &	2.24 & 0.004 & 2.329903976\\
            %\hline%
            20 & 50	& 2.195	& 0.051	& 2.204 &	0.051 &	2.158 &	0.028 &	2.217 &	0.02 &	2.06 &	0.338 &	2.318 &	0.003 & 2.298 & 0.003 & 	2.329903976\\
            %\hline%
            1 & 20 & 2.304 & 0.08 &	2.261 &	0.058 &	2.264 &	0.084 &	2.244 &	0.046 &	2.248 &	0.042 &	2.318 &	0.002 &	2.314 &	0.003 & 2.324202061\\
            %\hline%
            10 & 20 & 2.229 & 0.066 & 2.207 &	0.087 &	2.162 &	0.084 &	2.235 &	0.046 &	2.208 &	0.065 &	2.323 &	0.003 &	2.294 & 0.002 & 2.323732899\\
            %\hline%
            50 & 20 & 1.744 & 0.09 & 1.631 &	0.081 &	1.651 &	0.07 &	1.962 &	0.037 &	1.524 &	0.237 &	2.323 &	0.007 &	2.159 &	0.004 & 2.318082313\\
            %\hline%
            100	& 20 & 1.423 & 0.117 & 1.103 &	0.109 &	1.141 &	0.084 &	1.678 &	0.058 &	0	& 0	& 2.379	& 0.011 & 2.035 &	0.012 & 2.333104637\\
            \hline \hline \\
        \end{tabular}
        }
    \end{center}
\end{table*}

\begin{table*}[ht]
    \begin{center}
    \caption{Results Of CMI Estimation On Linear Model II}
    \label{table:cmi_lin_model2_comparison}
	\resizebox{\linewidth}{!}{
        \begin{tabular}{c c c c c c c c c c c c c c c c c}
            \hline
            \multirow{2}{*}{$d_z$} & \multirow{2}{*}{$n$} & \multicolumn{2}{c}{CGAN+Classifier} &  \multicolumn{2}{c}{CVAE+Classifier} &  \multicolumn{2}{c}{KNN+Classifier} &  \multicolumn{2}{c}{MI-Diff.+Classifier} &  \multicolumn{2}{c}{MI-Diff.+f-MINE} &  \multicolumn{2}{c}{C-MI-GAN} &   \multicolumn{2}{c}{MI-Diff-GAN} & True CMI\\
            \cline{3-16}
            & & Avg. & Std. &  Avg. & Std. &  Avg. & Std. &  Avg. & Std. &  Avg. & Std. &  Avg. & Std. & Avg. & Std. &\\
            \hline
            20	& 5	& 2.039	& 0.13 & 1.895 &	0.103 &	1.998 &	0.093 &	2.058 &	0.082 &	0.886 &	0.047 &	2.402 &	0.018 &	2.127 &	0.011 & 2.324819809\\
            %\hline%
            20	& 10 & 2.037 & 0.091 & 1.99 & 0.074 &	2.02 &	0.111 &	2.106 &	0.076 &	1.627 &	0.051 &	2.34 &	0.015 &	2.207 &	0.006 & 2.324819809\\
            %\hline%
            20 & 20	& 2.098	& 0.058	& 2.067 &	0.083 &	2.14 &	0.089 &	2.123 &	0.054 &	2.071 &	0.054 &	2.316 &	0.006 &	2.257 &	0.004 & 2.324819809\\
            %\hline%
            20	& 50 &	2.218 &	0.045 &	2.162 &	0.035 &	2.239 &	0.013 &	2.183 &	0.031 &	2.124 &	0.04 &	2.291 &	0.005 & 2.274 & 0.003 &	2.324819809\\
            %\hline%
            1 &	20 & 2.29 &	0.07 &	2.251 &	0.042 &	2.231 &	0.068 &	2.254 &	0.033 &	2.204 &	0.046 &	2.312 &	0.002 &	2.309 & 0.001 & 2.313665532\\
            %\hline%
            10 & 20	& 2.189	& 0.052	& 2.167 &	0.033 &	2.225 &	0.044 &	2.192 &	0.041 &	2.144 &	0.053 &	2.3 &	0.004 &	2.28 & 0.003 & 2.325076241\\
            %\hline%
            50 & 20	& 1.713	& 0.096	& 1.63 &	0.071 &	1.762 &	0.081 &	1.952 &	0.019 &	1.605 &	0.159 &	2.355 &	0.005 &	2.191 & 0.006 & 2.318283481\\
            %\hline%
            100	& 20 & 1.358 & 0.122	& 1.2 &	0.095 &	1.368 &	0.107 &	1.66 &	0.042 &	0.314 &	0.176 &	2.44 &	0.01 &	2.086 & 0.006 & 2.324413766\\
            \hline \hline \\
        \end{tabular}
        }
    \end{center}
\end{table*}

\begin{table*}[ht]
    \begin{center}
    \caption{ Results Of CMI Estimation On Non linear Model}
    \label{table:cmi_nonlin_comparison}
	\resizebox{\linewidth}{!}{
        \begin{tabular}{c c c c c c c c c c c c c c c c c}
            \hline
            \multirow{2}{*}{$d_z$} & \multirow{2}{*}{$n$} & \multicolumn{2}{c}{CGAN+Classifier} &  \multicolumn{2}{c}{CVAE+Classifier} &  \multicolumn{2}{c}{KNN+Classifier} &  \multicolumn{2}{c}{MI-Diff.+Classifier} &  \multicolumn{2}{c}{MI-Diff.+f-MINE} &  \multicolumn{2}{c}{C-MI-GAN} &  \multicolumn{2}{c}{MI-Diff-GAN} & True CMI\\
            \cline{3-16}
            & & Avg. & Std. &  Avg. & Std. &  Avg. & Std. &  Avg. & Std. &  Avg. & Std. &  Avg. & Std. & Avg. & Std. &\\
            \hline
            10 & 5 & 0.655 & 0.269 &  0.622 &	0.045 &	0.303 &	0.057 &	0.295 &	0.045 &	0.04 &	0.007 &	0.406 &	0.009  & 0.382 & 0.007 & 0.412\\
            %\hline%
            10 & 10 & 0.26 & 0.013 & 0.243 &	0.02 &	0.269 &	0.013 &	0.226 &	0.022 &	0.253 &	0.019 &	0.287 &	0.003  & 0.279 & 0.003 & 0.287\\
            %\hline%
            10 & 20 & 0.116 & 0.047 & 0.076 &	0.012 &	0.084 &	0.017 &	0.056 &	0.005 &	0.061 &	0.006 &	0.065 &	0.003 & 0.072 &	0.009 & 0.081\\
            %\hline%
            10 & 50 & 0.124 & 0.031 & 0.076 &	0.007 &	0.088 &	0.006 &	0.07 &	0.003 &	0.07 &	0.004 &	0.062 &	0.003 &	0.077 &	0.013 & 0.083\\
            %\hline%
            20 & 5 & 0.387 & 0.184 & 0.475 &	0.03 &	0.212 &	0.054 &	0.229 &	0.035 &	0.051 &	0.013 &	0.311 &	0.014 &	0.339 &	0.017 & 0.393\\
            %\hline%
            20 & 10 & 0.261 & 0.058 & 0.287 &	0.02 &	0.241 &	0.026 &	0.225 &	0.024 &	0.224 &	0.024 &	0.2 & 0.009 & 0.284 & 0.004 & 0.285\\
            %\hline%
            20 & 20 & 0.332 & 0.06 & 0.337 &	0.024 &	0.298 &	0.016 &	0.29 &	0.015 &	0.164 &	0.027 &	0.331 &	0.006 &	0.342 & 0.005 & 0.376\\
            %\hline%
            20 & 50 & 0.511 & 0.157 & 0.518 &	0.061 &	0.385 &	0.012 &	0.368 &	0.011 &	0.332 &	0.008 &	0.397 &	0.01 & 0.403 & 0.002 & 0.428\\
            %\hline%
            50 & 5 & 0 & 0 & 0.054 & 0.064 &	0.033 &	0.021 &	0.148 &	0.04 &	0.1 & 0.014 & 0.12 & 0.018 & 0.231 & 0.022 & 0.284\\
            %\hline%
            50 & 10 & 0.471 & 0.131 & 0.343 &	0.085 &	0.358 &	0.058 &	0.524 &	0.026 &	0.518 &	0.045 &	0.449 &	0.01 &	0.643 & 0.013 & 0.747\\
            %\hline%
            50 & 20 & 0.005 & 0.011 & 0 & 0 & 0 & 0 & 0.008 & 0.008 & 0.016 & 0.019 & 0.086 & 0.032 &	0.036 & 0.007 & 0.083\\
            %\hline%
            50 & 50 & 0.42 & 0.083 & 0.402 &	0.049 &	0.28 &	0.013 &	0.272 &	0.007 &	0.207 &	0.011 &	0.297 &	0.025 &	0.347 &	0.005 & 0.386\\
            %\hline%
            100 & 5 & 0 & 0 & 0 & 0 & 0 & 0	& 0.007 & 0.013	& 0.005	& 0.011	& 0	& 0.001	& 0.04 & 0.024 & 0.08\\
            %\hline%
            100	& 10 & 2.122 & 3.531 & 0 &	0 &	0 &	0 &	0.042 &	0.041 &	0.024 &	0.045 &	0.161 &	0.098 &	0.056 & 0.039 & 0.279\\
            %\hline%
            100 & 20 & 0 & 0 & 0.011 &	0.025 &	0.038 &	0.022 &	0.023 &	0.028 &	0.018 &	0.033 &	0.199 &	0.073 &	0.037 & 0.031 & 0.284\\
            %\hline%
            100	& 50 & 0 & 0 & 0 & 0 & 0 &	0 &	0 &	0 &	0.001 &	0.003 &	0.045 &	0.036 & 0.021 & 0.004 & 0.078\\
            %\hline%
            200	& 5	& 0	& 0	& 0	& 0	& 0	& 0	& 0.023	& 0.061	& 0.016	& 0.051	& 0.068	& 0.019	& 0.116 & 0.072 & 0.277\\
            %\hline%
            200	& 10 & 0 & 0 & 0 & 0 & 0 & 0 &	0.058 &	0.072 &	0 & 0 & 0.574	& 0.068 & 0.558 & 0.12 & 0.734\\
            %\hline*
            200 & 20 & 0 & 0 & 0 & 0	& 0	& 0	& 0.262	& 0.06	& 0 & 0 & 0.543 &	0.036 & 0.679 & 0.04 & 0.736\\
            %\hline&
            200 & 50 & 0 & 0 & 0 & 0 & 0 & 0 & 0.169 & 0.024 & 0 & 0 & 0.188 &	0.075 & 0.266 & 0.008 & 0.284\\
            \hline\hline\\
        \end{tabular}
        }
    \end{center}
\end{table*}
\begin{table*}[ht]
    \begin{center}
    \caption{CMI Estimation: RMSE Computed Over Different  Models}
    \label{table:cmi_rmse_estimators}
	\resizebox{\linewidth}{!}{
        \begin{tabular}{c c c c c c c c}
            \hline
            Models & CGAN+Classifier & CVAE+Classifier & KNN+Classifier & MI-Diff.+Classifier & MI-Diff.+f-MINE & C-MI-GAN & MI-Diff.-GAN\\
            \hline
            Linear Model I & $0.439951$ & $0.550609$ & $0.539955$ & $0.319044$ & $0.971969$ & $\boldsymbol{0.020806}$ &	$0.166482$\\
            %\hline%
            Linear Model II	& $0.439715$ & $0.519229$ & $0.430447$ & $0.311608$ & $0.952448$ & $\boldsymbol{0.053325}$ & $0.130044$\\
            %\hline%
            Non Linear Model & $0.49818$ & $0.290552$ & $0.289482$ & $0.228215$ & $0.241488$ & $0.119516$ & $\boldsymbol{0.099798}$\\
            \hline \hline \\
        \end{tabular}
        }
    \end{center}
\end{table*}